\documentclass[10pt]{article} % For LaTeX2e
\usepackage[T1]{fontenc}
\usepackage{newtxtext}
\usepackage[left=1.25in, top=1in, bottom=1in, right=1.25in]{geometry}

\usepackage{amsmath, amsfonts, bm}

\def\1{\bm{1}}

\def\vg{{\bm{g}}}

\def\vm{{\bm{m}}}

\def\vu{{\bm{u}}}
\def\vv{{\bm{v}}}

\def\vx{{\bm{x}}}
\def\vy{{\bm{y}}}

\def\vDelta{{\bm{\Delta}}}

\def\mA{{\bm{A}}}

\def\mH{{\bm{H}}}

\newcommand{\R}{\mathbb{R}}

\DeclareMathOperator{\Tr}{Tr}

\newcommand{\diag}{{\rm diag}}
\newcommand{\norm}[1]{\left\|#1\right\|}
\def\abs#1{\left| #1 \right|}

\newcommand{\inner}[2]{\left\langle #1,#2 \right\rangle}

\newcommand*{\E}{\mathbb{E}}

\newcommand{\eps}{\epsilon}

\newcommand{\tL}{\tilde{L}}
\newcommand{\tvx}{\tilde{\vx}}
\newcommand{\tvg}{\tilde{\vg}}
\newcommand{\tvm}{\tilde{\vm}}
\newcommand{\tv}{\tilde{v}}

\newcommand{\adam}{$\mathtt{Adam}$\xspace}
\newcommand{\adamini}{Adam-mini\xspace}
\newcommand{\adalayer}{Adalayer\xspace}
\newcommand{\poly}{\mathtt{poly}\xspace}

\newcommand{\rmsprop}{$\mathtt{RMSProp}$\xspace}

\newcommand{\sgd}{$\mathtt{SGD}$\xspace}
\newcommand{\signgd}{$\mathtt{SignGD}$\xspace}

\newcommand{\adasgd}{$\mathtt{AdaSGD}$\xspace}
\newcommand{\rotatedadam}{$\mathtt{rotated\ Adam}$\xspace}
\newcommand{\randperm}{$\mathtt{RandPerm}$\xspace}
\newcommand{\negligible}{\delta_T}

\usepackage[colorlinks,
            linkcolor=blue,
            anchorcolor=blue,
            citecolor=blue,
            urlcolor=black,
            backref=page
            ]{hyperref}       % hyperlinks
\usepackage{url}            % simple URL typesetting
\usepackage{nicefrac}       % compact symbols for 1/2, etc.
\usepackage[dvipsnames]{xcolor}  
\usepackage{amssymb, amsthm,mathtools}
\usepackage{cancel}
\usepackage{cleveref}
\usepackage{enumitem}
\usepackage{natbib}
\usepackage{thmtools} 
\usepackage{thm-restate}
\usepackage{caption}
\usepackage{subcaption}
\usepackage{nicefrac}
\usepackage{xfrac}
\usepackage{algorithms}
\usepackage{xspace}
\usepackage[disable]{todonotes}
\declaretheorem[name=Theorem, numberwithin=section]{theorem}
\declaretheorem[name=Lemma, sibling=theorem]{lemma}
\declaretheorem[name=Assumption, sibling=theorem]{assumption}
\declaretheorem[name=Definition, sibling=theorem]{definition}

\usepackage{sidecap}
\crefdefaultlabelformat{#2#1#3}
\creflabelformat{equation}{#2#1#3}
 
\usepackage{booktabs}
\title{Adam Exploits $\ell_\infty$-geometry of Loss Landscape via Coordinate-wise Adaptivity}
\date{}
\author{
Shuo Xie\qquad Mohamad Amin Mohamadi \qquad  Zhiyuan Li \\
Toyota Technological Institute at Chicago\\
\texttt{\{shuox,mohamadamin,zhiyuanli\}@ttic.edu} 
}

\begin{document}

\maketitle

\begin{abstract}
Adam outperforms SGD when training language models. Yet this advantage is not well-understood theoretically -- previous convergence analysis for Adam and SGD mainly focuses on the number of steps $T$ and is already minimax-optimal in non-convex cases, which are both $\widetilde{O}(T^{-1/4})$. In this work, we argue that the exploitation of nice $\ell_\infty$-geometry is the key advantage of Adam over SGD. More specifically, we give a new convergence analysis for Adam under novel assumptions that loss is smooth under $\ell_\infty$-geometry rather than the more common $\ell_2$-geometry, which yields a much better empirical smoothness constant for GPT-2 and ResNet models. Our experiments confirm that Adam performs much worse when the favorable $\ell_\infty$-geometry is changed while SGD provably remains unaffected. We also extend the convergence analysis to blockwise Adam under novel blockwise smoothness assumptions.
\end{abstract}
\section{Introduction}\label{sec:intro}
Large language models (LLMs) have gained phenomenal capabilities as their scale grows~\citep{radford2019language,kaplan2020scaling,brown2020language,  zhang2022opt,   
touvron2023llama,openai2023gpt,reid2024gemini}. However, pre-training LLMs is incredibly time-consuming. \adam~\citep{kingma2014adam} is the current to-go optimization algorithm for LLMs due to its fast convergence.
In contrast, \sgd, a popular and arguably the simplest optimizer, optimizes language model loss much more slowly than \adam. 

 However, the optimization benefit of \adam over \sgd cannot be explained by existing theory. Current convergence analyses for \adam and \sgd focus on the dependence on the number of steps under assumptions on the smoothness of the loss and gradient bounds~\citep{defossez2022simple}, and it has been shown that both \adam and \sgd achieve the minimax convergence rate $\widetilde{O}(T^{-1/4})$ in the non-convex settings~\citep{arjevani2023lower}. Thus according to the theory, in the worst case, \sgd would be more desirable compared to \adam because they have the same convergence rate, and yet \adam is less memory-efficient due to its coordinate-wise adaptivity, which needs to store the empirical moving average of second-order moments of past stochastic gradients. Therefore, we hypothesize that the coordinate-wise adaptivity in \adam is exploiting some unknown properties of LLMs which \sgd cannot make use of.

Towards this end, we identified a big difference between \adam and \sgd, which is ignored in previous works. That is, \sgd is rotation-equivariant, while \adam is only permutation equivariant (\Cref{def:rotation_invariance}). Intuitively, if we rotate the loss landscape, the trajectory of \sgd would be the same (up to some rotation), while the trajectory of \adam can be completely different. If \adam optimizes much more slowly after rotation, it suggests \adam is exploiting some non-rotation-invariant properties of the loss, which is not captured by standard smoothness assumptions in the convergence analysis.

\Cref{fig:adam_results} summarizes our findings by comparing \adam on the original and rotated loss. The performance of \adam on the rotated loss does become much worse than \adam on the original loss. We also test a memory-efficient and rotation-equivariant variant of \sgd, \adasgd~\citep{wang2020adasgd}, defined in \Cref{alg:adasgd}. 
Surprisingly, the rotated \adam performs even much worse than the \sgd variant. The results suggest it is impossible to explain the superior optimization performance of Adam over SGD just using rotation-invariant assumptions on the loss function, which raises the natural question, 
\begin{quote}
    % What are the non-rotation-invariant properties of a loss function that enable faster convergence of \adam than \sgd?
    \!\!\!\!\!\!\! \emph{What non-rotation-invariant properties of loss functions enable  \adam to converge faster than \sgd?}\!\!\!\!\!\!
\end{quote}

We hypothesize that the $\ell_2$-lipschitzness of loss gradient does not provide a tight-enough characterization of loss landscape of deep learning models in practice, such that we can separate \adam and other rotation-equivariant algorithms. 
Inspired by the similarity between \adam and \signgd and the fact that \signgd is the normalized steepest descent with respect to $\ell_\infty$-norm, we propose to use $\ell_\infty$-norm related smoothness as a better tool to analyze \adam. In particular, our main results use the $(1,1)$-norm of the Hessian of the loss normalized by variable dimension $d$, as the smoothness measure, instead of its spectral norm. And we prove a convergence rate of $O(\frac{1}{\sqrt{T}})$ for \adam without noise, or $O((\frac{\log T}{T})^{1/4})$ with noise. Our results have the same dependence on $T$ as previous results, but a much smaller smoothness constant when measured empirically. We empirically verify $(1,1)$-norm of Hessian positively correlates with final training loss of \adam on both synthetic tasks like quadratic loss and real tasks like training GPT2 on OpenWebText and ResNet on CIFAR10.

We summarize our contributions below:
% \paragraph{Contributions.}
\begin{enumerate}
    \item We show by experiments that the empirical optimization advantage of \adam over \sgd can not be explained solely under rotation-invariant assumptions. (\Cref{fig:adam_results})
    \item We propose a new complexity metric for the optimization problem, which is the $(1,1)$-norm of the Hessian matrix of loss, $\norm{\nabla^2 L(x)}_{1,1}$. We present a novel convergence result for \adam depending on this metric in the case of $\beta_1=0$. (\Cref{thm:main} ) 
    \item We further generalize the theoretical analysis ~(\Cref{thm:main_general_norm}) for \adam to blockwise \adam (\Cref{alg:blockwise_adam}) whose convergence rate can be characterized by a novel smoothness measure~(\Cref{defi:smoothness_to_partition}). \adam and \adasgd are two notable examples of blockwise \adam. In \adam, all blocks are of size $1$. In \adasgd, there is only one block. 

    \item We empirically verify that 
    when \adam converges more slowly on the rotated loss, the $(1,1)$-norm of Hessian also increases, which suggests that our new complexity metric for \adam's convergence is practically relevant.~(\Cref{sec:exp}).\footnote{The code is available at \url{https://github.com/mohamad-amin/adam-coordinate-adaptivity}. }
\end{enumerate}
\begin{figure}[t]
% \vspace{-1.4cm}
    \centering
    % \subfigure{
    %     \centering
\includegraphics[width=0.48\textwidth]{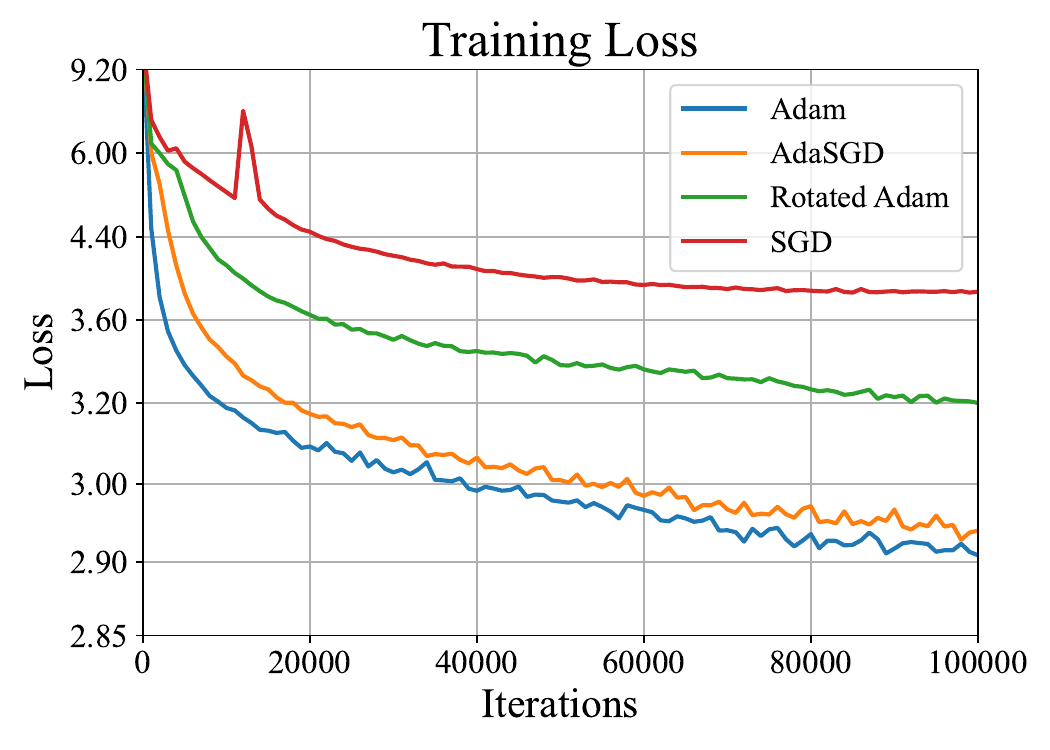}
    % }
    \hfill
    % \subfigure{
    %     \centering
\includegraphics[width=0.48\textwidth]{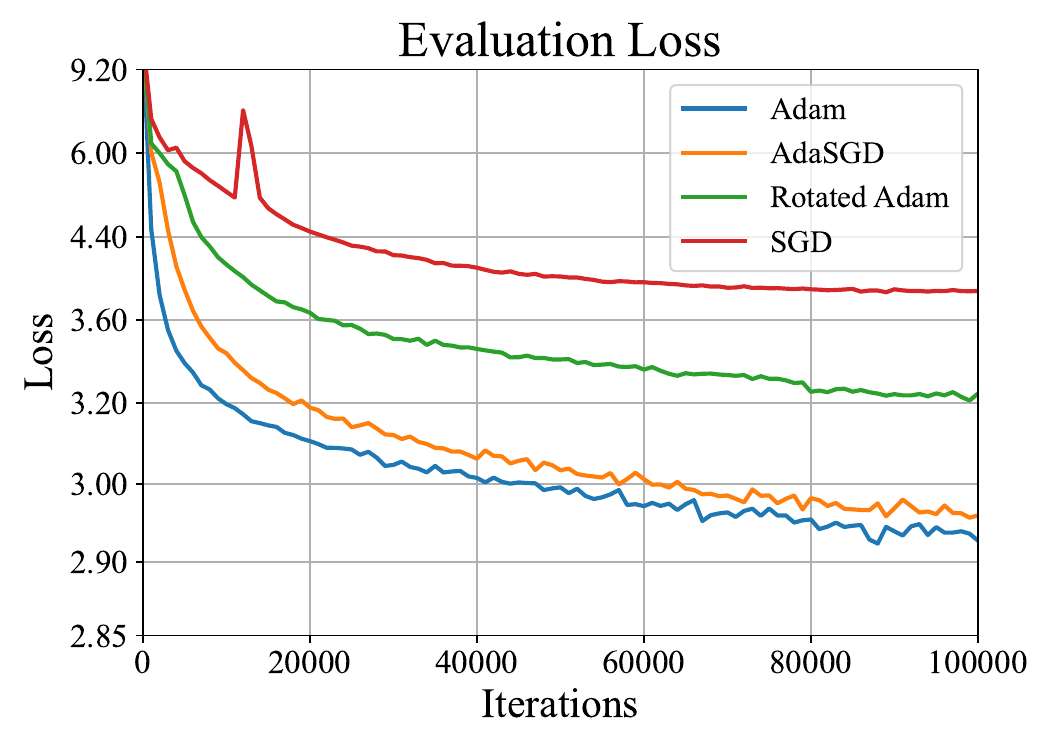}
% \hfill
% \includegraphics[width=0.32\textwidth]{figs/gpt2_eval_loss.pdf}
    % }
    \caption{Training and evaluation losses of \adam, \adasgd and \sgd on GPT-2. \rotatedadam means running \adam on a rotated loss. 
    \adam on the original loss converges the fastest as expected. But convergence of \adam on a rotated loss is much slower, notably even worse than \adasgd. }
    \label{fig:adam_results}
    % \vspace{-0.3cm}
\end{figure}

\section{Preliminaries}\label{sec:notation}\label{sec:prelim}
\paragraph{Notations.} For $\vx\in\mathbb{R}^d$, we define the vector $p$-norm $\|\vx\|_p$ as $(\sum_{i=1}^d x_i^p)^{1/p}$ for $p\in[1,\infty]$. For a matrix $\bm{A} \in \mathbb{R}^{d_1 \times d_2}$, its $(1,1)$-norm $\|\bm{A}\|_{1,1}$ is defined as $\sum_{i=1}^{d_1} \sum_{j=1}^{d_2} \abs{A_{i,j}}$  and its operator norm induced by vector $p$-norm $\|\cdot\|_p$ as $\sup_{\vx\in\mathbb{R}^d}\frac{\|\bm{A}\vx\|_q}{\|\vx\|_p}$, denoted by $\|\bm{A}\|_p$, where $\frac{1}{q}+\frac{1}{p}=1$ and $\|\cdot\|_q$ is the dual norm of $\|\cdot\|_p$. For a square matrix $\bm{A} \in \mathbb{R}^{d \times d}$, $\abs{\bm{A}}$ is defined as the unique square root of $\bm{A}^\top \bm{A}$. 
For a deterministic loss function $L(\vx)$, we consider optimization over $L$ with access only to independent stochastic functions $\{L_t(\vx) \}_{t=1}^T$ such that $\E L_t(\vx) = L(\vx)$ for any input $\vx \in \mathbb{R}^d$. 

\paragraph{Rotation.}
For an invertible function $\mathcal{T}: \mathbb{R}^d \rightarrow \mathbb{R}^d$, $\mathcal{T}$ is a rotating transformation if there exists an orthogonal matrix $\bm{T} \in \mathbb{R}^{d \times d}$ such that $\mathcal{T}(\vx) = \bm{T} \vx$. $\mathcal{T}$ is a permutating transformation if there exists a permutation $\pi : [d] \rightarrow [d]$ such that $\mathcal{T}(\vx) = [x_{\pi(1)}, \dots, x_{\pi(d)}]^\top$. A permutating transformation is always a rotating transformation. We will use $\mathcal{R}$ to denote a rotating transformation. 
\begin{figure}[t]
    \centering
    \begin{minipage}{0.48\textwidth}
        \begin{algorithm}[H]
            \caption{Adam}\label{alg:adam}
            \begin{algorithmic}
                \HYPER{$\beta_1, \beta_2, \epsilon \geq 0$, total steps $T$, learning rate $\{\eta_t\}_{t=1}^T$, $\epsilon$, initial $\vm_0, v_0$}
                \INPUT{initialization $\vx_0$, stochastic losses $\{L_t\}_{t=1}^T$} 
                \State $v_{0,i} \gets v_0$
                \For{$t=1, 2, \cdots, T$}
                    \State $g_{t,i} \gets \nabla_i L_t(\vx_{t-1})$
                    \State $m_{t,i} \gets \beta_1 m_{t-1,i} + (1-\beta_1) g_{t,i}$
                    \State $v_{t,i} \gets \beta_2 v_{t-1,i} + (1-\beta_2) g_{t,i}^2$
                    \State $x_{t,i} \gets x_{t-1,i} -\eta_t \frac{m_{t,i}}{\sqrt{v_{t,i}+\epsilon}}$
                \EndFor
                \State \Return $\vx_T$
            \end{algorithmic}
        \end{algorithm}
    \end{minipage}
    \hfill
    \begin{minipage}{0.48\textwidth}
    \vspace{-15pt}
        \begin{algorithm}[H]
            \caption{AdaSGD}\label{alg:adasgd}
            \begin{algorithmic}
                \HYPER{$\beta_1, \beta_2, \epsilon\geq 0$, total steps $T$, learning rate $\{\eta_t\}_{t=1}^T$, initial $\vm_0, v_0$}
                \INPUT{initialization $\vx_0$, stochastic losses $\{L_t\}_{t=1}^T$}
                % \State $\vm_0 \gets \vg_1, v_0 \gets v_0$  
                \For{$t=1, 2, \cdots, T$}
                    \State $g_{t,i} \gets \nabla_i L_t(\vx_{t-1})$
                    \State $m_{t,i} \gets \beta_1 m_{t-1,i} + (1-\beta_1) g_{t,i}$\footnotemark
                    \State $v_t \gets \beta_2 v_{t-1} + (1-\beta_2) (\norm{\vg_t}_2^2/d)$
                    \State $x_{t,i} \gets x_{t-1,i} -\eta_t\frac{m_{t,i}}{\sqrt{v_t+\epsilon}}$
                \EndFor
                \State \Return $\vx_T$
            \end{algorithmic}
        \end{algorithm}
    \end{minipage}
\end{figure}
\footnotetext{Here it is slightly different from \cite{wang2020adasgd}. We use an exponential average of the gradient for $\vm_t$ instead of momentum. Our definition makes \adasgd the same as \adam in a one-dimensional problem.}
\begin{definition}\label{def:rotation_invariance}
For initialization $\vx_0$ and stochastic losses $\{L_t\}_{t=1}^T$, we can get $\vx_t$ when running algorithm $A$ on $(\vx_0, \{L_t\}_{t=1}^T)$. For a transformation $\mathcal{T}$, we can also get $\tilde{\vx}_t$ when running $A$ with the same hyperparameters on $(\tilde{\vx}_0 , \{\tL_t\}_{t=1}^T)$ with $\tilde{\vx}_0 = \mathcal{T}^{-1} (\vx_0)$ and $\tL_t =L_t \circ \mathcal{T}$. 

An algorithm $A$ is \underline{equivariant w.r.t. $\mathcal{T}$} if it always holds that $\tilde{\vx}_t = \mathcal{T}^{-1}(\vx_t)$ for any hyperparameters, initialization and stochastic losses. An algorithm $A$ is \underline{rotation-equivariant} if it is equivariant w.r.t. any rotating transformation $\mathcal{R}$. And $A$ is \underline{permutation-equivariant} if it is equivariant w.r.t. any permutating transformation. 
\end{definition}
The following \Cref{thm:invariance} shows the difference between \adam and \adasgd, whose proof is in \Cref{sec:invariance_property}. We provide a visual example in \Cref{fig:adam_adasgd}. It also shows the similarity between \signgd and \adam. 
\begin{SCfigure}[1.2][htbp]
% \vspace{-1.4cm}
    % \centering
    % \subfigure{
    %     \centering
\includegraphics[width=0.6\textwidth]{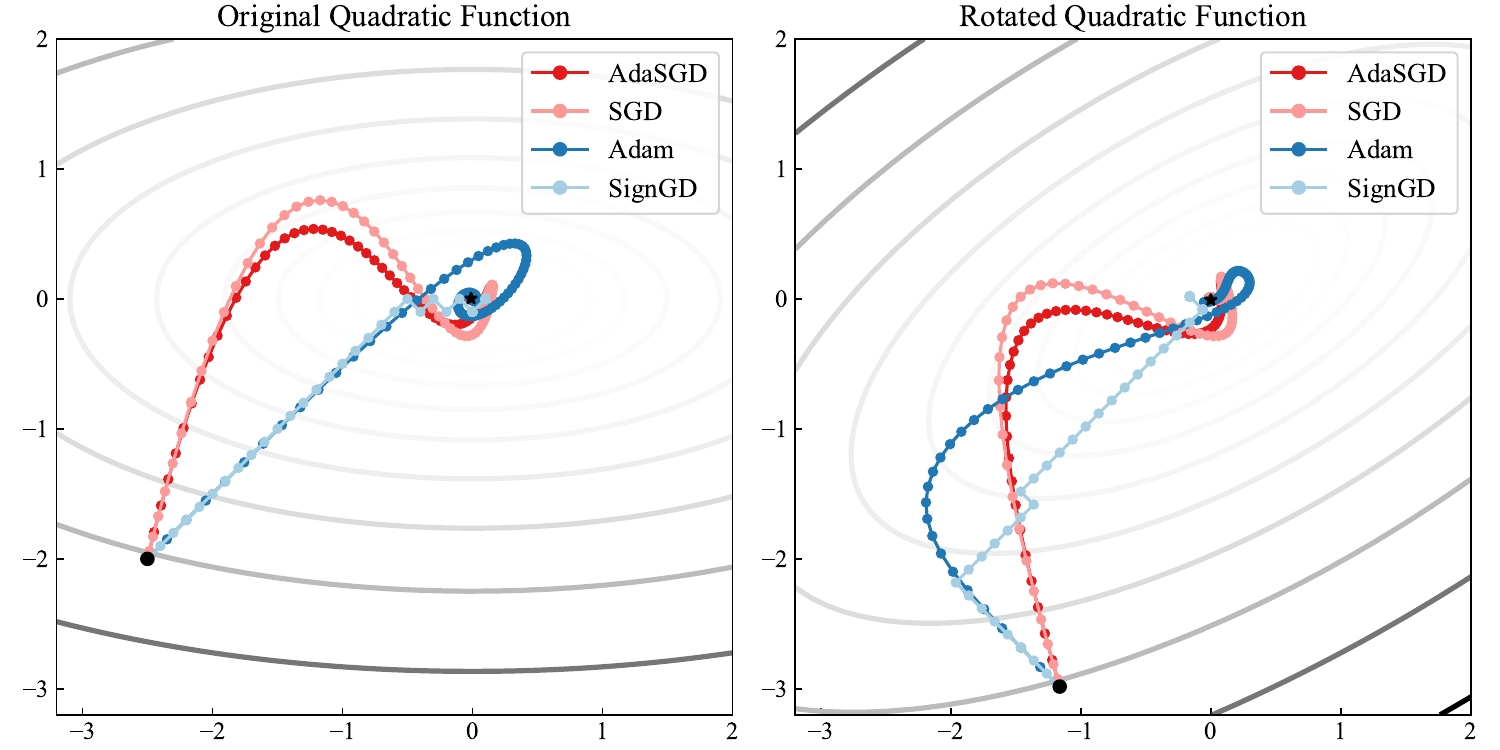}
    \caption{Optimization trajectories of different algorithms on a quadratic function before (left) and after (right) rotating the coordinate system. The starting point is marked in black. \adasgd and \sgd exhibit rotation-equivariant behavior, as their trajectories rotate consistently with the objective function. In contrast, \adam and \signgd fail to preserve trajectory shape under rotation, demonstrating that they are not rotation-equivariant.
 }
    \label{fig:adam_adasgd}
    % \vspace{-0.3cm}
\end{SCfigure}
\begin{restatable}{theorem}{rotation}\label{thm:invariance}\label{thm:rotation_effect}
\sgd and \adasgd are rotation-equivariant. \adam and \signgd are only permutation-equivariant. 
\end{restatable}

\section{Main Results: Convergence Rates of \texorpdfstring{\adam}{Adam}}\label{sec:main_results}

In this section, we present our main theoretical results, starting with a convergence analysis of \adam for stochastic smooth losses with coordinate-wise gradient noise~(\Cref{thm:main}). 
We allow non-convex losses and thus the convergence is measured by the $\ell_1$ norm of the gradient. 
For the deterministic loss, our best convergence rate (\Cref{thm:convergence_rate_signgd}) is achieved by \signgd (\adam with $\beta_1=\beta_2=0$).
For the stochastic loss with bounded gradient noise variance, our best rate (\Cref{cor:stochastic}) is achieved by \rmsprop (\adam with $\beta_1=0$ and $\beta_2\in[0,1]$). 

Then we extend our analysis of \adam to more general blockwise \adam~(\Cref{thm:main_general_norm}), which contains both \adam and \adasgd as special cases. We also come up with novel smoothness measures~(\Cref{ass:general_norm_smoothness}) corresponding to the set of blocks used in blockwise \adam.

Similar to previous works~\citep{defossez2022simple}, our analysis could be extended to the most general case of \adam, where both $\beta_1$, $\beta_2$ are non-zero. But the rate becomes strictly worse than \rmsprop (the case of $\beta_1=0$), as there will be some extra polynomials of $\frac{1}{1-\beta_1}$. 
We decide not to include the result for the most general case, on the one hand for ease of presentation, and on the other hand, because such result can not explain the optimization benefit of momentum ($\beta_1>0$) in practice and does not add any insight on the benefit of \adam. We hypothesize that we are missing some important features of loss landscape of transformers in the theoretical assumptions and we leave this for future work. 

\subsection{Warmup: \texorpdfstring{\signgd (\adam with $\beta_1 = \beta_2=0$)}{SignGD}}

In this section, we present the convergence analysis for \signgd as a warm-up and illustrate how \adam could benefit from a non-rotation-invariant property of the loss landscape, which in particular is the $\ell_\infty$ smoothness. The key observation here is that \signgd is the normalized steepest descent with respect to $\ell_\infty$ norm \citep{xie2024implicit}, so it is more natural to analyze its convergence using $\ell_\infty$ norm geometry of the loss. 

\begin{definition}\label{defi:smoothness_to_specific_norm}
    Given a norm $\norm{\cdot}$ on $\mathbb{R}^d$ and $\norm{\cdot}_*$ as its dual norm, we say a function $L$ is $H$-smooth w.r.t. $\norm{\cdot}$ if for any $\vx,\vy\in\mathbb{R}^d$, we have that $\norm{\nabla L(\vx)-\nabla L(\vy)}_* \leq H \norm{\vx-\vy}$.
\end{definition}
\Cref{thm:convergence_rate_signgd} gives the convergence rate for \signgd and the proof is in \Cref{sec:proof_signgd}. 
\begin{restatable}{theorem}{signgdrate}\label{thm:convergence_rate_signgd}
    Let $L$ be $H$-smooth w.r.t. $\ell_\infty$ norm and $\{\vx_t\}_{t=1}^T$ be the iterates of \signgd (\adam with $\beta_1=\beta_2=0$) on $L$ with initialization $\vx_0$ and learning rate $\eta$, it holds that
    \begin{equation*}
    \min_{1 \leq t \leq T} \norm{\nabla L(\vx_t)}_1 \leq \frac{L(\vx_0) - \min L(\vx)}{T\eta} + \frac{H\eta}{2}.
    \end{equation*}
    If we choose $\eta = \sqrt{\frac{2(L(\vx_0)-\min L(\vx))}{TH}}$, then $\min_{1 \leq t \leq T} \norm{\nabla L(\vx_t)}_1\le \sqrt{\frac{2H(L(\vx_0)-\min L (\vx))}{T}}$.
\end{restatable}

\subsection{Main result: \texorpdfstring{\rmsprop (\adam with $\beta_1=0,\beta_2\in[0,1]$)}{RMSProp}}\label{subsec:main_adam}
It is well-known that \signgd might not converge in the stochastic case as the expectation of descent direction for mini-batch loss may not be a descent direction for $L$. \rmsprop is proposed to address this issue by using a moving average of the squared gradient per coordinate to reduce the correlation between the denominator and the numerator, thus making the expected update direction less biased~\citep{hinton2012neural}. In this subsection we formalize the above intuition and show indeed a positive $\beta_2$ in \adam helps convergence in the stochastic case. The main challenges here are from both lower bounding the first-order term and upper bounding the second-order term in the modified descent lemma (the \rmsprop counterpart of \Cref{eq:descent_lemma_nsd}). 
\begin{equation*}
    L(\vx_t)-L(\vx_{t-1}) \leq -\eta_t \nabla L(\vx_t)^\top \frac{\vg_t}{\sqrt{\vv_t+\epsilon}} + \frac{H}{2} \eta_t^2 \norm{\frac{\vg_t}{\sqrt{\vv_t+\epsilon}}}_\infty^2
\end{equation*}
We can only upper bound $\norm{\frac{\vg_t}{\sqrt{\vv_t+\epsilon}}}_\infty^2$ by $\frac{1}{1-\beta_2}$ without more fine-grained analysis on the relationship between gradients in each step, which will greatly hurt the dependence of convergence rate on $1-\beta_2$. However, even though the update $\frac{g_{t,i}}{\sqrt{v_{t,i}+\epsilon}}$ at step $t$ can be as large as $\frac{1}{\sqrt{1-\beta_2}}$ with some very large $g_{t,i}$, the average moving speed for each coordinate should be close to $1$. Therefore, we introduce \Cref{ass:lipshitz}, which is slightly stronger than \Cref{defi:smoothness_to_specific_norm} but allows us to decompose the second order term into each coordinate according to \Cref{lem:second_order_general_norm}. It also facilitates the coordinate-wise analysis for the first order term. We note this definition also appears in Assumption 2.3 of the concurrent work \citet{maladkar2024convergence}.
\begin{restatable}{definition}{lipschitz}
    \label{ass:lipshitz}
     For any  $\mathbf{H} = (H_1,\ldots,H_d)\in\mathbb{R}^d$, we say a function $L$ is \emph{$\mathbf{H}$-smooth coordinate-wisely w.r.t. $\ell_\infty$ norm }, if and only if $\abs{\nabla_i L(\vx) - \nabla_i L(\vy)} \leq H_i \norm{\vx-\vy}_\infty$ for any $i\in [d]$, $\vx, \vy \in \mathbb{R}^d$. 
 \end{restatable}

 \paragraph{(1,1)-norm as a surrogate complexity measure for coordinate-wise smoothness.} $H_i$ in \Cref{ass:lipshitz} is determined by $\sup_{\vx} \sum_{j=1}^d \abs{\nabla^2_{i,j} L(\vx)}$, which is difficult to compute because it requires taking supreme over the entire domain. A computationally-tractable alternative is to approximate $\sum_{i=1}^d H_i$ locally by the $(1,1)$-norm of Hessian of loss along the training trajectory. We are the first to provide an efficient approximation algorithm with concentration guarantees in \Cref{subsec:exp_matrix_norm_estimation}, which uses hessian-vector product against random Cauchy vectors. 
 
By definition, $\mathbf{H}$-smoothness coordinate-wisely w.r.t. $\ell_\infty$ norm implies $\sum_{i=1}^d H_i$-smoothness w.r.t. $\ell_\infty$ norm. 
We also need \Cref{ass:bounded_noise} to measure the influence of noise in the stochastic setting.
 \begin{restatable}[Coordinate-wise noise]{assumption}{noise}\label{ass:bounded_noise}
    There exist constants $\sigma_i$ such that 
    $\E \left[ \nabla_i L_t(\vx)-\nabla_i L(\vx)\right]^2 \leq \sigma_i^2$ for any $i \in [d]$, $t \in \mathbb{N}$ and $\vx \in \mathbb{R}^d$. 
\end{restatable}
 We present the main result in \Cref{thm:main}. The sketch of the proof is presented in \Cref{sec:proof_sketch} and the complete proof for the generalized blockwise Adam algorithm is presented in \Cref{sec:proof_detail}. The proof incorporates some key steps from \cite{li2024frac}, extending them to accommodate the generalized algorithm and different smoothness assumptions. 
\begin{theorem}[Main, \adam]\label{thm:main}
 Let $\{L_t\}_{t=1}^T$ be independent stochastic losses satisfying \Cref{ass:bounded_noise} and that their expectation $L$ is $\mH$-coordinate-wisely smooth w.r.t. $\ell_\infty$ norm. For \adam with $\beta_1=0$, we have that
 \begin{align*}
\min_{\frac{T}{2} < t \leq T} \E \norm{\nabla L(\vx_t)}_1 &\leq O \left(E + \sqrt{E} \sqrt{\frac{\beta_2^{\frac{T}{4}}}{T(1-\beta_2)} d v_0 + \sum_{i=1}^d \sigma_i + d \sqrt{\epsilon}} \right)
\end{align*}
with 
\begin{align*}
    E&=\frac{2}{\eta T} \E \left[L(\vx_0)-L(\vx_T) \right] 
        +\left( 1+ \frac{\beta_2 F }{T(1-\beta_2)} \right)\left( \eta \sum_{i=1}^d H_i + \sqrt{1-\beta_2} \sum_{i=1}^d \sigma_i\right)
\end{align*} and 
$$F = 2\ln\left(1+ \frac{ \sum_{i=1}^d \sigma_i^2 + \norm{\nabla L(\vx_0)}_\infty^2 + \sum_{i\in[d]} H_i^2 \eta^2 T(T + \frac{1}{1-\beta_2})} {v_0+\epsilon}\right)+\ln{32}. $$
\end{theorem}

We can determine the convergence rate of RMSprop by choosing appropriate hyperparameters $\eta$ and $\beta_2$ in \Cref{thm:main} to minimize $E$. We would assume that $v_0+\epsilon > (\sum_{i=1}^d \sigma_i^2 + \| \nabla L(\vx_0) \|_\infty^2 + \sum_i H_i^2 \eta^2)/\poly(T)$ and $\frac{1}{1-\beta_2}=\poly(T)$. Then we can simplify the term by considering $F = O(\log T)$.

The two terms involving $\sum_{i=1}^d \sigma_i $ have a lower bound $\Theta \left(\sum_{i=1}^d \sigma_i \left(\frac{\log{T}}{T} \right)^\frac{1}{2} \right)$, which is achieved with $1-\beta_2=\Theta \left(\frac{\log{T}}{T} \right)$.
Then the three terms involving $\eta$ has a lower bound $\Theta \left( \sqrt{\frac{\left(L(\vx_0)-\min_\vx L(\vx) \right) \sum_{i=1}^d H_i}{T}}\right)$ reached by $\eta=\Theta\left(\sqrt{\frac{L(\vx_0)-\min_\vx L(\vx) }{T\sum_{i=1}^d H_i}}\right)$. These hyperparameter choices yield the optimal convergence rate for stochastic case in \Cref{cor:stochastic}.  For convenience, we define $
  R\triangleq \left( L(\vx_0)-\min_\vx L(\vx) \right) \sum_{i=1}^d H_i$, which will be the core term in \Cref{cor:stochastic,cor:deterministic}.
\begin{restatable}[Stochastic Case, general $\sigma_i$]{corollary}{stochastic}\label{cor:stochastic}
Let $\{L_t\}_{t=1}^T$ be independent stochastic losses satisfying \Cref{ass:bounded_noise} and that their expectation $L$ is $\mH$-coordinate-wisely smooth w.r.t. $\ell_\infty$ norm. For $\beta_1=0$, $1-\beta_2=\Theta(\frac{\log{T}}{T})$, $\epsilon=0$, $\eta=\Theta\left(\sqrt{\frac{L(\vx_0)-\min_\vx L(\vx) }{T\sum_{i=1}^d H_i}}\right)$ and $v_0> (\sum_{i=1}^d \sigma_i^2 + \| \nabla L(\vx_0) \|_\infty^2 + \sum_i H_i^2 \eta^2)/\poly(T)$, we have that
\begin{align*}
    \min_{\frac{T}{2} < t \leq T} \E \norm{\vg_t}_1 =O\left( \sqrt{\frac{R}{T}} +  \sqrt{\sum_{i=1}^d \sigma_i } \left(\frac{R}{T}\right)^\frac{1}{4} + \sum_{i=1}^d \sigma_i \left(\frac{\log{T}}{T} \right)^\frac{1}{4} + \negligible\right)
    % \vspace{-0.3cm}
\end{align*}
with $\negligible=\sqrt{\frac{d v_0}{T(1-\beta_2)}} \exp{\left(-\frac{T(1-\beta_2)}{8} \right)}\left[\left(\frac{R}{T}\right)^\frac{1}{4} + \sqrt{\sum_{i=1}^d \sigma_i} \left(\frac{\log{T}}{T} \right)^\frac{1}{4} \right]$.
\end{restatable}
Here $\negligible$ can be smaller than any polynomial of $T$ by manipulating the value of $\frac{T(1-\beta_2)}{\log{T}} = \Theta(1)$. Then $\sum_{i=1}^d \sigma_i \left( \frac{\log{T}}{T} \right)^\frac{1}{4}$ is the leading term w.r.t. $T$ in the rate whose coefficient only involves $\sum_{i=1}^d \sigma_i$. It suggests that the rate can be much improved when noise is small. Below we get the convergence rate with the same hyperparameters in deterministic case in \Cref{cor:deterministic}.

\begin{restatable}[Deterministic Case, $\sigma_i=0$]{corollary}{deterministic}\label{cor:deterministic}\label{cor:deterministic_optimal}
Let $\{L_t\}_{t=1}^T$ be deterministic losses satisfying \Cref{ass:bounded_noise} and that their expectation $L$ is $\mH$-coordinate-wisely smooth w.r.t. $\ell_\infty$ norm. For $\beta_1=0$, $1-\beta_2=\Omega(\frac{\log{T}}{T})$, $\epsilon=0$, $\eta=\Theta\left(\sqrt{\frac{L(\vx_0)-\min_\vx L(\vx)}{T \sum_{i=1}^d H_i}} \right)$ and $v_0> (\sum_{i=1}^d \sigma_i^2 + \| \nabla L(\vx_0) \|_\infty^2 + \sum_i H_i^2 \eta^2)/\poly(T)$ for any polynomial $\poly(T)$, we have that
\begin{align*}
    \min_{\frac{T}{2} < t \leq T} \norm{\vg_t}_1 =O\left( \sqrt{\frac{R}{T}} + \negligible\right)
\end{align*}
with $\negligible=\sqrt{\frac{d v_0}{T(1-\beta_2)}} \exp{\left(-\frac{T(1-\beta_2)}{8} \right)} \left(\frac{R}{T} \right)^\frac{1}{4}$. 
\end{restatable}

\Cref{cor:deterministic_optimal} almost recovers \Cref{thm:convergence_rate_signgd}, except for the smoothness constant. Specifically, it uses $\sup_{\vx}\|\nabla^2 L(\vx)\|_{1,1}$ which is larger than $\sup_{\vx}\|\nabla^2 L(\vx)\|_\infty$ in \Cref{thm:convergence_rate_signgd} as $\|\cdot\|_{1,1}\ge \|\cdot\|_{\infty}$ always holds. This gap is due to technical difficulty of analyzing \adam as mentioned in the beginning of \Cref{subsec:main_adam}.
\paragraph{Dependence on \texorpdfstring{$\epsilon$, $v_0$ and $\beta_2$}{hyperparameters}. }
While many previous works rely on the relatively large magnitude of $\epsilon$ compared to $\vv_t$ and give a bound in the regime of \sgd when the adaptive effect is dominated by the constant $\epsilon$~\citep{zaheer2018adaptive,de2018convergence}, our result actually prefers $\epsilon$ to be $0$ while maintaining the value of 
$v_0+\epsilon$. 
We also note the dependence of our bound in \Cref{thm:main} on $v_0$ is very mild and logarithmic. \Cref{thm:main} has similar convergence rates for all $v_0$ of magnitude at most $\poly(T)$, while most previous result only addresses the case where $v_{0,i}$ is at the scale of noise~\citep{li2024frac} or $0$. The main reason for this adaptivity to a wide range of $v_0$ is our specific choice of $\beta_2=  1-\Theta(\frac{\log T}{T})$, which allows the initial large $v_0$ to decay fast and resume normal training. Other existing results using $\beta_2 = 1-\Theta(1/T)$~\citep{defossez2022simple,li2024frac} cannot allow large initial value $v_0$ because $v_0$ only decays a constant fraction throughout the training and the effective learning rate will be too small.

\subsection{A unified analysis for blockwise \texorpdfstring{\adam}{Adam}}\label{sec:general_theory}
In this subsection, we present convergence analysis for a broader class of adaptive algorithms defined in \Cref{alg:blockwise_adam}. It can be viewed as a coarser version of \adam because it does pre-conditioning blockwisely instead of coordinate-wisely.  Since \adam and \adasgd can be viewed as special cases of blockwise \adam (\Cref{alg:blockwise_adam}) with $\Phi_{\mathtt{Adam}}:i\mapsto i$ and $\Phi_{\mathtt{AdaSGD}}:i\mapsto 1$ respectively, any convergence results for \Cref{alg:blockwise_adam} would imply convergence of \adam and \adasgd. Finally we also note that such blockwise \adam has been recently studied empirically by some concurrent work, where the algorithm is named by \adamini~\citep{zhang2024adam} and \adalayer~\citep{zhao2024deconstructing}. 
\begin{algorithm}[ht]
    \caption{Blockwise \adam}\label{alg:blockwise_adam}
    \begin{algorithmic}
        \HYPER{$\beta_1, \beta_2, \epsilon \geq 0$, block partition $\Phi:[d]\to[B]$, total steps $T$, learning rate schedule $\{\eta_t\}_{t=1}^T$, $\epsilon$, initial $\vm_0$, $v_0$. }
        \INPUT{initialization $\vx_0$, stochastic losses $\{L_t\}_{t=1}^T$}
        \State $v_{0,b} \gets v_0$ 
        \For{$t=1, 2, \cdots, T$}
        % \State $t \gets t+1$
        \State $g_{t,i} \gets \nabla_i L_t(\vx_{t-1})$
        \State $m_{t,i} \gets \beta_1 m_{t-1,i} + (1-\beta_1) g_{t,i}$
        \State $v_{t, b} \gets \beta_2 v_{t-1,b} + (1-\beta_2) \left(\sum_{\Phi(i)=b} g_{t,i}^2 \right)/d_b$
        \State $x_{t,i} \gets x_{t-1,i} -\eta_t \frac{m_{t,i}}{\sqrt{v_{t, \Phi(i)}+\epsilon}}$
        \EndFor
        \State \Return $\vx_T$
    \end{algorithmic}
\end{algorithm}

 We first introduce more notations. For a partition function $\Phi:[d] \rightarrow [B]$ where $B$ is the number of blocks, $(b)$ is defined as $\Phi^{-1}(b) = \{i | \Phi(i)=b \}$ and $d_b = \# (b)$, the number of parameters in block $b$. We define the vector $\vx_{(b)}$ as $[x_i]_{\Phi(i)=b}$ and the submatrix $\mathbf{A}_{(b), (b')}$ as $[A_{i,j}]_{\Phi(i)=b, \Phi(j) =b'}$.
\begin{definition}[$\Phi$-norm]
    We define the $(\infty,2)$-norm w.r.t. partition $\Phi$ of the vector $\vx$ as the $\ell_\infty$ norm of the vector $\left(\frac{\norm{\vx_{(b)}}_2}{\sqrt{d_b}}\right)_{b=1}^B$, which is $\max_{b\in [B] }\frac{\norm{\vx_{(b)}}_2}{\sqrt{d_b}}$. For convenience, we will denote it by $\norm{\vx}_\Phi$ or just call it $\Phi$-norm. We denote its dual norm by $\|\vx\|_{\Phi,*}$, which is equal to $\sum_{b=1}^B \sqrt{d_b} \norm{\vx_{(b)}}_2$. 
\end{definition}

% \zhiyuan{@shuo: we should cite Bernstein et al., SIGN SGD paper. I think their assumption 2 is the first to introduce this kind of smoothness. We should also mention this paper ``Adaptive Gradient Methods for Constrained Convex Optimization
% and Variational Inequalities'' and two concurrent works: AdaGrad under Anisotropic Smoothness, and Convergence Analysis of Adaptive Gradient Methods under Refined
% Smoothness and Noise Assumptions. (they are all for adagrad) We should say sth like the previous measure cannot be measured empically because it involves solving some very high dimensional psd optimization problem, we are the first the actually empirically measure this kind of smoothness, though only through an upper bound.} \shuo{I don't know why we still want to include \Cref{ass:lipshitz} now. We can replace it with the exact assumption in signsgd paper. }

We will need \Cref{defi:smoothness_to_partition} to characterize the smoothness of the loss function when analyzing general blockwise \adam. We note $\Phi_\mathtt{Adam}$-smoothness also appears in Assumption 2 in \cite{bernstein2018signsgd} to analyze \signgd. Previous and concurrent work~\citep{ene2021adaptive,liu2024adagrad,jiang2024convergence} employ the same assumption to analyze AdaGrad. We are the first to analyze \adam under such smoothness assumption. Moreover, we are the first to empirically measure it for loss functions in real task with theoretical guarantee. 
\begin{definition}[$\Phi$-smoothness]\label{defi:smoothness_to_partition}
We say a diagonal matrix $\mA$ follows partition $\Phi$ if and only if the diagonal elements of $\mA$ are constant within each block defined by $\Phi$, i.e., there are $a_1, \cdots, a_B\in\mathbb{R}$, such that $\mA_{i,i} = a_{\Phi(i)}$ for any $i \in [d]$. 

We say a twice-differentiable $L$ is $H$-smooth under partition $\Phi$ if there exists a diagonal matrix $\mA$ following partition $\Phi$ such that $H=\Tr(\mA)$ and $\mA$ dominates $\abs{\nabla^2 L(\vx)}$ for all $x$. 
We further define the $\Phi$-smoothness of loss $L$, denoted by $H(L,\Phi)$, as the smallest constant $H$ such that $L$ is $H$-smooth under parition $\Phi$, that is, 
% {\color{violet}$H(L,\Phi)=\min_{\mA \text{ follows } \Phi, \mA \succeq\nabla^2 L(\vx) \text{ for all }\vx} \Tr(\mA)$}. 
\begin{equation}\label{eq:phi_smoothness}
\color{violet} H(L, \Phi) = \min_{\substack{\mA \text{ follows } \Phi \\ \mA \succeq \abs{\nabla^2 L(\vx)}, \forall \vx}} \Tr(\mA)
\end{equation}
\end{definition}

We note that $\Phi$-smoothness is different from the smoothness under $\Phi$-norm, where the latter is equal to $\sup_{\vx \in \mathbb{R}^d} \sup_{\norm{\vu}_\Phi \leq 1} \norm{\nabla^2 L(\vx) \vu}_{\Phi, *}$. For each $\vx$, it holds that 
\begin{align*}
\sup_{\norm{\vu}_\Phi \leq 1} \norm{\nabla^2 L(\vx) \vu}_{\Phi, *} =  \sup_{\norm{\vu}_\Phi \leq 1} \sup_{\norm{\vv}_\Phi \leq 1} \vv^\top\nabla^2 L(\vx) \vu = \sup_{\norm{\vu}_\Phi \leq 1} \abs{\vu^\top\nabla^2 L(\vx) \vu}.
\end{align*}
When $\mA \succeq \abs{\nabla^2 L(
\vx)}$, we have that $\sup_{\norm{\vu}_\Phi \leq 1} \abs{\vu^\top\nabla^2 L(\vx) \vu} \leq \sup_{\norm{\vu}_\Phi \leq 1} \vu^\top\mA \vu$. When diagonal $\mA$ follows partition $\Phi$, we have that $\sup_{\norm{\vu}_\Phi \leq 1} \vu^\top\mA \vu = \Tr(\mA)$. 
So $\Phi$-smoothness defined in \Cref{defi:smoothness_to_partition} is always no smaller than the smoothness under $\Phi$-norm. 

Another advantage with the definition is that the $H(L,\Phi)$ will always non-increase with a more fine-grained partition. 
For two partition functions $\Phi_1$ and $\Phi_2$, we say $\Phi_1$ includes $\Phi_2$ if and only if $\Phi_1(i)=\Phi_1(j)$ for any $i,j \in [d]$ such that $\Phi_2(i)=\Phi_2(j)$. If a diagonal matrix $\mA$ follows partition $\Phi_1$ and partition $\Phi_1$ includes $\Phi_2$, then $\mA$ also follows $\Phi_2$. 
So we have that $\{\mA \mid \mA \text{ follows } \Phi_1, \mA \succeq \abs{\nabla^2 L(\vx)} \text{ for all }\vx\} \subseteq \{\mA \mid \mA \text{ follows } \Phi_2, \mA \succeq \abs{\nabla^2 L(\vx)} \text{ for all }\vx\}$ and $H(L,\Phi_1)$ is no smaller than $H(L, \Phi_2)$. Since $\Phi_{\mathtt{AdaSGD}}$ includes any partition and any partition includes $\Phi_{\mathtt{Adam}}$, $H(L,\Phi_{\mathtt{AdaSGD}})$ is the largest and $H(L,\Phi_{\mathtt{Adam}})$ is the smallest among all the partitions $\Phi$. 

With generalized \Cref{ass:general_norm_noise} on noise, we can prove the result for blockwise \adam in \Cref{thm:main_general_norm}. 
\begin{assumption}[Generalized version of \Cref{ass:bounded_noise}]\label{ass:general_norm_noise}
There exists constant $\sigma_b$ such that
\begin{align*}
 \E \norm{ \nabla_{(b)} L_t(\vx) - \nabla_{(b)} L(\vx)}_2^2 \leq d_b \sigma_b^2   
\end{align*}
for any block $b\in [B], t\in \mathbb{N}$ and $\vx \in \mathbb{R}^d$.
\end{assumption}

\begin{restatable}[Main, Blockwise \adam]{theorem}{maingeneral}\label{thm:main_general_norm}
For a specific partition $\Phi$, we consider the updates defined in \Cref{alg:blockwise_adam}. 
Under \Cref{ass:general_norm_noise}, we have that 
\begin{align*}
\min_{\frac{T}{2} < t \leq T} \E \sum_{b=1}^B \sqrt{d_b} \norm{\nabla_{(b)} L(\vx_t)}_2&\leq 2\sqrt{2} E + \sqrt{2E} \sqrt{\frac{4\beta_2^{\frac{T}{4}}}{T(1-\beta_2)} d\sqrt{v_{0}} + \sum_{b=1}^B d_b \sigma_b + d \sqrt{\epsilon}}
\end{align*}
with 
\begin{align*}
    E&=\frac{2}{\eta T} \E \left[L(\vx_0)-L(\vx_T) \right]  
        +\left( 1  +\frac{\beta_2 F}{T(1-\beta_2)}\right) \left( \eta {\color{violet}H(L,\Phi)} + 2\sqrt{1-\beta_2} \sum_{b=1}^B d_b \sigma_b\right),
\end{align*}
and 
\begin{align*}
    F=2\ln\left(1+ \frac{ \sum_{b=1}^B \sigma_b^2 + \norm{\nabla L(\vx_0)}_{\Phi}^2 + \sum_{b\in[B]} H_b^2 d_b \eta^2 T(T + \frac{1}{1-\beta_2})} {v_0+\epsilon}\right) + \ln{32}. 
\end{align*}
\end{restatable}
In order to obtain \Cref{thm:main} from \Cref{thm:main_general_norm}, we rely on the following definition that generalizes \Cref{ass:lipshitz}. 
\begin{definition}[Generalized version of \Cref{ass:lipshitz}]\label{ass:general_norm_smoothness}
 For any partition function $\Phi:[d] \rightarrow [B]$ and $\mH = (H_1,\ldots,H_B)\in\mathbb{R}^B$, we say a function $L$ is \emph{$\mathbf{H}$-blockwisely-smooth w.r.t. $\Phi$-norm }, iff for any $b\in [B]$, $\vx, \vy \in \mathbb{R}^d$, 
    \begin{align*}
         \norm{\nabla_{(b)} L(\vx) - \nabla_{(b)} L(\vy) }_2 \leq \sqrt{d_b}H_b \norm{\vx-\vy}_{\Phi}
    \end{align*}
\end{definition}
The following \Cref{lem:second_order_general_norm} will show that $H(L,\Phi)$ is upper bounded by $\sum_{b=1}^B d_b H_b$ when $L$ is $\mH$-blockwisely-smooth w.r.t. $\Phi$-norm. The proof of \Cref{lem:second_order_general_norm} is deferred in \Cref{sec:proof_detail}. 
\begin{restatable}{lemma}{secondorder}\label{lem:second_order_general_norm}
For any twice differentiable loss which is $\mathbf{H}$-blockwisely-smooth w.r.t. $\Phi$-norm (\Cref{ass:general_norm_smoothness}), we have for any $\vx$ and $\vDelta \in \mathbb{R}^d$,
        \begin{align}
            \abs{\vDelta^\top \nabla^2 L(\vx) \vDelta} \leq \sum_{b=1}^B H_b \norm{\vDelta_{(b)}}_2^2.
        \end{align}
Then the diagonal matrix $\mA$ defined by $\mA_{i,i} = H_{\Phi(i)}$ follows partition $\Phi$ and always dominates $\abs{\nabla^2 L(\vx)}$. 

\end{restatable}

When $L$ is $(H_1, \dots, H_d)$-smooth coordinate-wisely w.r.t. $\ell_\infty$ norm, it satisfies $\diag(H_1, \dots, H_d) \succeq \abs{\nabla^2 L(\vx)}$ for all $\vx \in \mathbb{R}^d$ and $\diag(H_1, \dots, H_d)$ follows $\Phi_{\mathtt{Adam}}$ partition. 
So $H(L,\Phi_{\mathtt{Adam}})$ is at most $\sum_{i=1}^d H_i$ and we can directly derive \Cref{thm:main} from \Cref{thm:main_general_norm}. 
For $\Phi_{\mathtt{AdaSGD}}$ being the mapping $i \mapsto 1$, $H(L,\Phi_{\mathtt{AdaSGD}})$ is the same as the smoothness under $\Phi_{\mathtt{AdaSGD}}$-norm, whose value is $d \sup_{\vx \in \mathbb{R}^d}\norm{\nabla^2 L(\vx)}_2$. 

\paragraph{Different norms for smoothness. } As an implication of \Cref{thm:main_general_norm}, we can get analogs of \Cref{cor:deterministic,cor:deterministic_optimal,cor:stochastic} for \adasgd, with the corresponding noise and smoothness assumptions.
When the optimization is not noise-dominated and $\sqrt{\frac{R}{T}} = \sqrt{\frac{(L(\vx_0)-\min_{\vx} L(\vx))H(L,\Phi)}{T}}$ becomes the leading term, the choice of $\Phi$ now matters a lot. 
The key difference between \adasgd and \adam lies in the gap between $H(L,\Phi_{\mathtt{AdaSGD}})$ and $H(L,\Phi_{\mathtt{Adam}})$ and comparing these coefficients can provide insight into which algorithm may be more effective under different conditions.

Previous analyses of \adam's convergence \citep{shi2021rmsprop,defossez2022simple,li2024frac} usually assume smoothness under the $\ell_2$ norm and the rate of \adam ends up being identical to the rate of \adasgd, which fails to explain why \adam often performs better than \adasgd in practice. By adopting an $\ell_\infty$ norm smoothness assumption, the coefficient for \adam's convergence rate changes from $d \sup_{\vx}\|\nabla^2 L(\vx)\|_2$ to $H(L,\Phi_{\mathtt{Adam}})$, where the latter is typically much smaller when \adam optimizes faster because $\sup_\vx \norm{\nabla^2 L(\vx)}_{1,1}$ is much smaller than $d \sup_\vx \norm{\nabla^2 L(\vx)}_2$. 

Finally, we note that the $\Phi_\mathtt{Adam}$-smoothness $H(L,\Phi_{\mathtt{Adam}})$ is not rotation-invariant in the sense that $H(L,\Phi_{\mathtt{Adam}})\neq H(L\circ \mathcal{R},\Phi_{\mathtt{Adam}})$ for a typical rotation $\mathcal{R}$. In practice, the $(1,1)$-norm of Hessian matrix can vary a lot when a rotation is performed on the loss as shown in \Cref{sec:exp_quadratic}. In contrast, $\Phi_\mathtt{AdaSGD}$-smoothness $H(L,\Phi_{\mathtt{AdaSGD}})$ is invariant under loss rotations.

\subsection{Proof sketch of \texorpdfstring{\Cref{thm:main_general_norm}}{Theorem 3.11}}\label{sec:proof_sketch}
We will use $\bar{\vg}_t= \E [\vg_t|\vx_{<t}]=\nabla L(\vx_{t-1})$ to denote the full batch gradient. We start by considering the decrease of $L(\vx_t)$ in a single step $t$. We can upper bound the second order term in the Taylor expansion (\Cref{eq:taylor}) with $\mH \succeq \nabla^2 L(\vx)$ that achieves $\Tr(\mH) = H(L, \Phi)$. Then we can get 
\begin{align}
    L(\vx_{t})-L(\vx_{t-1})
        &\le -\eta \sum_{i=1}^d \frac{g_{t,i} \bar{g}_{t,i}}{\sqrt{v_{t,\Phi(i)}+\epsilon}} +\frac{1}{2}\eta^2 \sum_{b=1}^B H_b \frac{\norm{\vg_{t,(b)}}_2^2 }{v_{t,b}+\epsilon} \label{eq:taylor}\\
        &
        \le -\eta \sum_{b=1}^B \frac{\vg_{t,(b)}^\top \bar{\vg}_{t,(b)}}{\sqrt{v_{t,b}+\epsilon}} +\frac{1}{2}\eta^2 \sum_{b=1}^B H_b d_b \frac{\norm{\vg_{t,(b)}}_2^2/d_b }{v_{t,b}+\epsilon} \label{eq:taylor2}
\end{align}
The proof contains two main parts: lower bounding the first order term using $\|\bar{\vg}_t\|_{\Phi,*}$ and upper bounding the second order term. Below we address the second-order term first. As mentioned in \Cref{subsec:main_adam}, the second order term $\frac{\norm{\vg_{t,(b)}}_2^2/d_b}{v_{t,b}+\epsilon}$ can be as large as $\frac{1}{1-\beta_2}$ for one step. But we can employ \Cref{lem:momentum_ratio} to bound the sum by $T + \frac{\beta_2}{1-\beta_2} \ln{\frac{v_{T,b}+\epsilon}{v_{0,b}+\epsilon}}$ rather than $\frac{T}{1-\beta_2}$, where we set $v_t\triangleq v_{t,b}$ and $g_t \triangleq \norm{\vg_{t,(b)}}_2/\sqrt{d_b}$. 
\begin{restatable}[]{lemma}{momentumratio}\label{lem:momentum_ratio}
    Given any $0 <\beta_2 <1$, for any scalar sequences $\{v_t\}_{t=0}^T$ and $\{g_t\}_{t=1}^T$ satisfy that $v_0 \geq 0, v_1 >0$ and $v_t-\beta_2 v_{t-1} \geq (1-\beta_2)g_t^2$ for $t \geq 1$, it holds that 
    \begin{equation}
        \sum_{t=1}^{T}\frac{g_t^2}{v_t} \leq T + \frac{\beta_2}{1-\beta_2}\ln{\frac{v_T}{v_0}}.
    \end{equation}
\end{restatable}

Now we turn to the first term. Ideally, for each block $b\in[B]$, we would like to connect the first order term to $\|\bar \vg_{t,(b)}\|_2^2$ by taking expectation, i.e., $\mathbb{E}_t  \frac{\vg_{t,(b)}^\top \bar{\vg}_{t,(b)}}{\sqrt{v_{t,b}+\epsilon}} \approx  \frac{\mathbb{E}_t \vg_{t,(b)}^\top \bar{\vg}_{t,(b)}}{\sqrt{v_{t,b}+\epsilon}} = \frac{\|\bar \vg_{t,(b)}\|_2^2}{\sqrt{v_{t,b}+\epsilon}}$, where we use $\mathbb{E}_t [\cdot]$ as abbreviation for $\mathbb{E}[\cdot|\vx_{<t}]$. However, this is not correct because both the numerator and denominator in $\frac{\vg_{t,(b)}^\top \bar{\vg}_{t,(b)}}{\sqrt{v_{t,b}+\epsilon}}$ depend on the stochastic gradient $\vg_t$. To circumvent this difficulty, we lower bound each conditional expectation $\mathbb{E}_t \frac{\vg_{t,(b)}^\top \bar{\vg}_{t,(b)}}{\sqrt{v_{t,b}+\epsilon}}$ by $\frac{\mathbb{E}_t \vg_{t,(b)}^\top \bar{\vg}_{t,(b)}}{2\sqrt{\mathbb{E}_t v_{t,b}+\epsilon}}$, minus error terms related to noise magnitude $\sigma_{b}$.
We can further have $\frac{\mathbb{E}_t \vg_{t,(b)}^\top \bar{\vg}_{t,(b)}}{\sqrt{\mathbb{E}_t v_{t,b}+\epsilon}}  \ge  \frac{\norm{\bar{\vg}_{t,(b)}}_2^2}{\sqrt{\Tilde{v}_{t,b}+\epsilon}}$, where $\Tilde{v}_{t,b} = \beta_2 v_{t-1,b} +(1-\beta_2) \left( \norm{\bar{\vg}_{t,(b)}}_2^2/d_b + \sigma_b^2\right)$. This leads to \Cref{lem:first_order_approx_rmsprop} whose proof is deferred to \Cref{sec:proof_detail}. 

\begin{restatable}[first-order approximation]{lemma}{firstorder}\label{lem:first_order_approx_rmsprop} 
With \Cref{ass:general_norm_noise}, it holds that for any block $b \in [B]$
\begin{align}
\E \sum_{t=1}^T \frac{\vg_{t,(b)}^\top \bar{\vg}_{t,(b)}}{\sqrt{v_{t,b}+\epsilon}} &\geq \frac{1}{2} \E \sum_{t=1}^T \frac{\norm{\bar{\vg}_{t,(b)}}_2^2}{\sqrt{\tilde{v}_{t,b}+\epsilon}}
-\sqrt{1-\beta_2}T d_b \sigma_b  - \frac{d_b \sigma_b \beta_2}{\sqrt{1-\beta_2}} \E \left[\ln{\frac{v_{T,b}+\epsilon}{v_{0,b}+\epsilon}} \right]. 
\end{align}
\end{restatable}

Combining \Cref{lem:momentum_ratio,lem:first_order_approx_rmsprop} and \Cref{eq:taylor2} yields an upper bound for $\E \sum_{t=1}^T \sum_{b=1}^B  \frac{\norm{\bar{\vg}_{t,(b)}}_2^2}{\sqrt{\tilde{v}_{t,b}+\epsilon}}$ involving initial suboptimality gap, noise and smoothness.
Finally, we employ Cauchy inequality~(\Cref{eq:cauchy_sketch}) and \Cref{lem:denominator} to upper bound $\norm{\bar{\vg}_{t} }_{\Phi, *}$ using $\E \sum_{t=\frac{T}{2}+1}^T \sum_{b=1}^B  \frac{\norm{\bar{\vg}_{t,(b)}}_2^2}{\sqrt{\tilde{v}_{t,b}+\epsilon}}$. This completes the proof.

\begin{align}\label{eq:cauchy_sketch}
    \sum_{t=\frac{T}{2}+1}^T\norm{\bar{\vg}_t}_{\Phi, *}=\sum_{t=\frac{T}{2}+1}^T \sum_{b=1}^B \sqrt{d_b} \norm{\bar{\vg}_{t,(b)}}_2 \leq \sqrt{\sum_{t=\frac{T}{2}+1}^T\sum_{b=1}^B \frac{\norm{\bar{\vg}_{t,(b)}}_2^2}{\sqrt{\Tilde{v}_{t,b}+\epsilon}}} \sqrt{\sum_{t=\frac{T}{2}+1}^T \sum_{b=1}^B d_b \sqrt{\Tilde{v}_{t,b}+\epsilon}}. 
\end{align}

\begin{restatable}{lemma}{denominator}\label{lem:denominator}
With \Cref{ass:general_norm_noise}, it holds that for any block $b \in [B]$
    \begin{align}
        \sum_{t=\frac{T}{2}+1}^T \E \left[\sqrt{\tilde{v}_{t,b}+\epsilon} \right] &\leq \frac{2\beta_2^\frac{T}{4}}{1-\beta_2} \sqrt{v_{0,b}} + \frac{T}{2} \sigma_b + \frac{T}{2} \sqrt{\epsilon} + 2 \sum_{t=1}^T \E \left[\frac{\norm{\bar{\vg}_{t,(b)}}_2^2/d_b}{\sqrt{\tilde{v}_{t,b}+\epsilon}} \right]. 
    \end{align}
\end{restatable}
\section{Experiments}\label{sec:exp}

In order to empirically investigate and confirm the implications of our proposed theory, we compare the training performance of \adam with \adasgd, \sgd and \rotatedadam on multiple different tasks. In \adam and its variants (including AdaSGD) we set $(\beta_1, \beta_2) = (0.9, 0.99)$ for experiments on the quadratic loss~(\Cref{sec:exp_quadratic}) and ResNet18~(\Cref{sec:exp_resnet}) and $(\beta_1, \beta_2) = (0.9, 0.95)$ for experiments on GPT-2~(\Cref{sec:exp_gpt2}). Momentum in SGD is also set to $0.9$. Weight decay is always deactivated.  

\subsection{Quadratic loss}\label{sec:exp_quadratic}
We perform controlled experiments on quadratic loss to study the relationship between optimization speed of \adam and the shape of Hessian in terms of $\Phi_{\mathtt{Adam}}$-smoothness.
More specifically, we consider $\Sigma=\text{diag}(\underbrace{1, \cdots, 1}_{10}, 1, \frac{1}{2^2}, \frac{1}{3^2}, \cdots, \frac{1}{990^2}) \in \mathbb{R}^{1000 \times 1000}$ and manually generate orthogonal matrices $\mathcal{R}_i$ in the following way. We first sample $\mathcal{M} \in \mathbb{R}^{d \times d}$ where $\mathcal{M}_{i,j}$ is i.i.d. sampled from $N(0,1)$. Then $\mathcal{A} = \mathcal{M} - \mathcal{M}^\top$ is a skew-symmetric matrix and $\exp{(t\mathcal{A})}$ represents a continuous family of matrices. We define $\mathcal{R}_i = \exp{(t_i \mathcal{A})}$ for different $t_i$. When $t_i=0$, we know $\mathcal{R}_i=I$. When $t_i\to \infty$, $\mathcal{R}_i$ converges to a random orthogonal matrix in distribution. We pick $t_1=0.002, t_2=0.008, t_3=0.015, t_4=0.1$ for our experiments. 

Then we optimize $L_0(\vx) = \frac{1}{2}\vx^\top \Sigma \vx$ with \adasgd and \adam and optimize $L_i(\vx) = \frac{1}{2} \vx^\top \mathcal{R}_i^\top \Sigma \mathcal{R}_i \vx$ with \adam for $100$ steps. Because \adasgd is rotation-equivariant, the optimization process of \adasgd is the same on all $L_i$. The initial $\vx_0$ is decided by sampling from $\text{Unif}([0,1]^{1000})$ when there is no rotation. When $\mathcal{R}_i$ is not identity matrix, we will start training from $\mathcal{R}_i^\top \vx_0$ to ensure the initial loss values are the same across different rotations. We tune learning rates for each setting with $10$ random seeds and present their lowest average loss with standard deviation in \Cref{tab:quadratic}. 

We find a clear pattern that \adam optimizes worse when the $(1,1)$-norm of Hessian matrix increases, as suggested by our \Cref{cor:deterministic}. Moreover, when $(1,1)$-norm divided by dimension is smaller than spectral norm, \adam tends to optimize faster than \adasgd, as suggested by our \Cref{thm:main_general_norm}. 
\begin{table}[ht]
\centering
% \vspace{-1.3cm}
% \begin{tabular}{l| c c c}
%     \toprule
%    \textbf{Optimizer} & \textbf{$(1, 1)$-norm$/d$} & \textbf{Loss} ($\beta_1 = \beta_2 = 0$) & \textbf{Loss} ($\beta_1 = 0.9, \beta_2 = 0.99$) \\
%     \midrule
%     AdaSGD      & $0.00582$ & $0.00881$ & $0.00172$ \\
%     Adam        & $0.00582$ & $0.00030$ & $0.00001$ \\
%     Adam ($\mathcal{R}_1$) & $0.04162$ & $0.00317$ & $0.00062$ \\
%     Adam ($\mathcal{R}_2$) & $0.25364$ & $0.00588$ & $0.00122$ \\
%     Adam ($\mathcal{R}_3$) & $0.61866$ & $0.00747$ & $0.00179$ \\
%     Adam ($\mathcal{R}_4$) & $1.29959$ & $0.00920$ & $0.00239$ \\
%     \bottomrule
% \end{tabular}
\begin{tabular}{l| c c c}
    \toprule
   \textbf{Optimizer} & \textbf{$(1, 1)$-norm$/d$} & \textbf{Loss} ($\beta_1 = \beta_2 = 0$) & \textbf{Loss} ($\beta_1 = 0.9, \beta_2 = 0.99$) \\
    \midrule
    AdaSGD      & $0.01164$ & $0.00887 \pm 0.00119$ & $0.00405 \pm0.00021$ \\
    Adam        & $0.01164$ & $0.00022 \pm 0.00007$ & $0.00002 \pm 0.00001$ \\
    Adam ($\mathcal{R}_1$) & $0.08324$ & $0.00314 \pm 0.00031$ & $0.00066 \pm 0.00008$ \\
    Adam ($\mathcal{R}_2$) & $0.50729$ & $0.00567 \pm 0.00053$ & $0.00134 \pm 0.00007$ \\
    Adam ($\mathcal{R}_3$) & $1.23731$ & $0.00751 \pm 0.00086$ & $0.00183 \pm 0.00009$ \\
    Adam ($\mathcal{R}_4$) & $2.59919$ & $0.00978 \pm 0.00132$ & $0.00254 \pm 0.00008$ \\
    \bottomrule
\end{tabular}
\caption{The final loss values obtained by different optimizers and the $(1,1)$-norm of Hessian matrix for the corresponding unrotated objective and rotated objectives. The spectral norm of the Hessian matrix is always $1$. \adam optimizes worse when the $(1,1)$-norm of Hessian matrix increases, as suggested by our \Cref{cor:deterministic}. Moreover, when $(1,1)$-norm divided by $d$ is much smaller than spectral norm, \adam tends to optimize faster than \adasgd, which justifies the effectiveness of $\Phi$-smoothness as a tool to predict the optimization speed of blockwise \adam. %\zhiyuan{(maybe) Todo: add  $\beta_1= 0,\beta_2 =0.99$; 2. explore larger $t$?}\shuo{Shall we change the (1,1)-norm by $\Phi$-smoothness? It might seem confusing as the reviewers pointed out. }
}
\label{tab:quadratic}
% \vspace{-0.3cm}
\end{table}
\subsection{GPT-2 on language modeling task}\label{sec:exp_gpt2}
We train GPT-2 small (124M parameters)\footnote{Our codebase is built upon \texttt{nanoGPT} codebase \url{https://github.com/karpathy/nanoGPT}. } on the OpenWebText corpus containing more than 9B tokens for 100k iterations with sequence length of $512$ sequence length and $480$ sentences per batch. We use cosine learning rate schedule of the same peak learning rate $6\times10^{-4}$ for all the adaptive optimizers, which is also the default of \texttt{nanoGPT} codebase. We did a grid search to find the maximum possible peak learning rate for \sgd\footnote{We tried $0.00001, 0.00003, 0.0001, 0.0003, 0.001, 0.003, 0.01, 0.03, 0.1, 0.3, 1$. }. 
The training losses and evaluation losses of different optimizers are plotted in \Cref{fig:adam_results}.
As mentioned in \Cref{sec:intro}, \adam converges faster than \adasgd while they both converge faster than \rotatedadam. 
Since we propose the $(1,1)$-norm of Hessian as a non-rotation-invariant metric that can affect the convergence rate of \adam, we also measure it for the original loss function $L$ and rotated loss function $\tilde{L}$ on checkpoints trained with different losses. The results are presented in \Cref{tab:norm}. 
The same correlation between norms and convergence rates holds here. The smaller the norm is, the faster the optimizer works.

\begin{table}[h]
\centering
% \vspace{-1.2cm}
\begin{tabular}{c c c c }
    \toprule
    \textbf{Optimizer} & \textbf{Smoothness Metric} &\textbf{Upper Bound} & \textbf{Estimated Value}  \\
    \midrule
    AdaSGD   & $H(L,\Phi_{\mathtt{AdaSGD}})$ & $d \norm{\nabla^2 L(\vx)}_2$& $4.2446$  \\
    Adam  & $H(L,\Phi_{\mathtt{Adam}})$ & $\norm{\nabla^2 L(\vx)}_{1,1}$& $2.3538$  \\
    Rotated Adam  & $H(L\circ \mathcal{R},\Phi_{\mathtt{Adam}})$ & $\norm{R^\top\nabla^2 L(\vx) R}_{1,1}$& $14.3745$  \\
    \bottomrule
\end{tabular}
\caption{Hessian norms for the last GPT-2 checkpoints trained with different optimizers.}\label{tab:norm}
\end{table}

We also explore how learning rate can affect the performance of different optimizers by using peak learning rate $3 \times 10^{-4}$ and $1.8 \times 10^{-3}$ for all the adaptive optimizers. The training losses are plotted in \Cref{fig:lr_results}. All the optimizers perform worse with smaller learning rate, which aligns with the common understanding that optimizers tend to work better with larger learning rate as long as the training is still stable. When a larger learning rate is used, the performance of \adam is improved but the performance of  rotated \adam becomes worse than with the default learning rate. The training with \adasgd even completely failed. This suggests that another advantage of \adam over \adasgd: it can maintain stable training at a larger learning rate, which is often beneficial to faster and more efficient convergence.  

\begin{figure}[t]
    \centering
\includegraphics[width=0.32\textwidth]{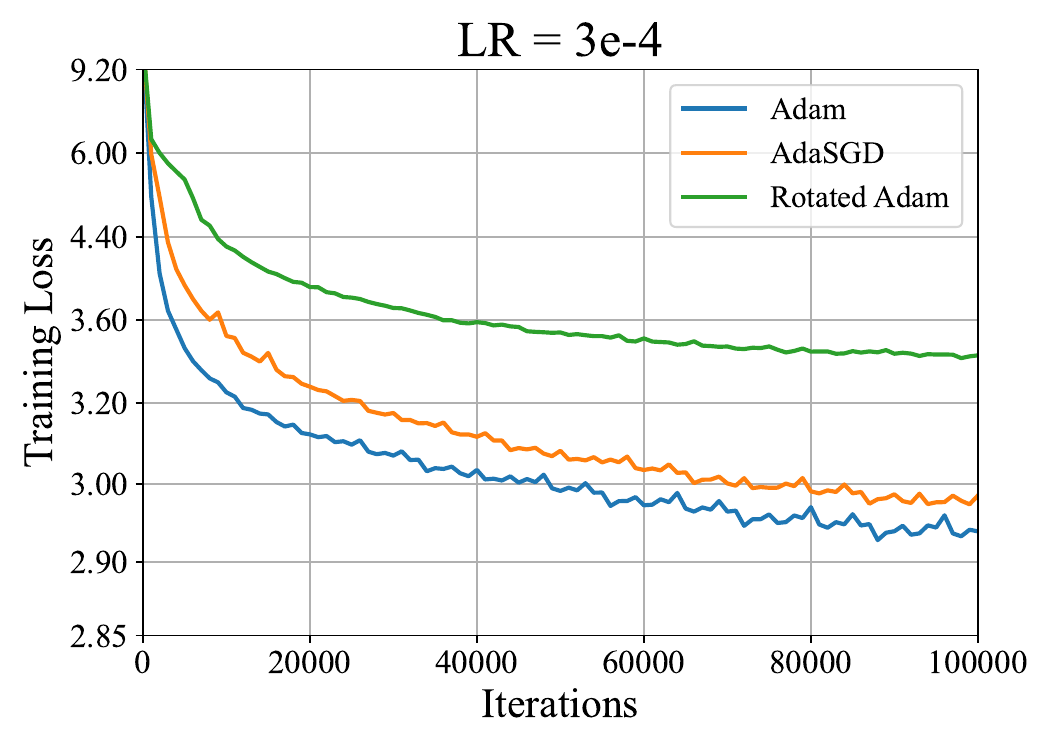}
    \hfill
\includegraphics[width=0.32\textwidth]{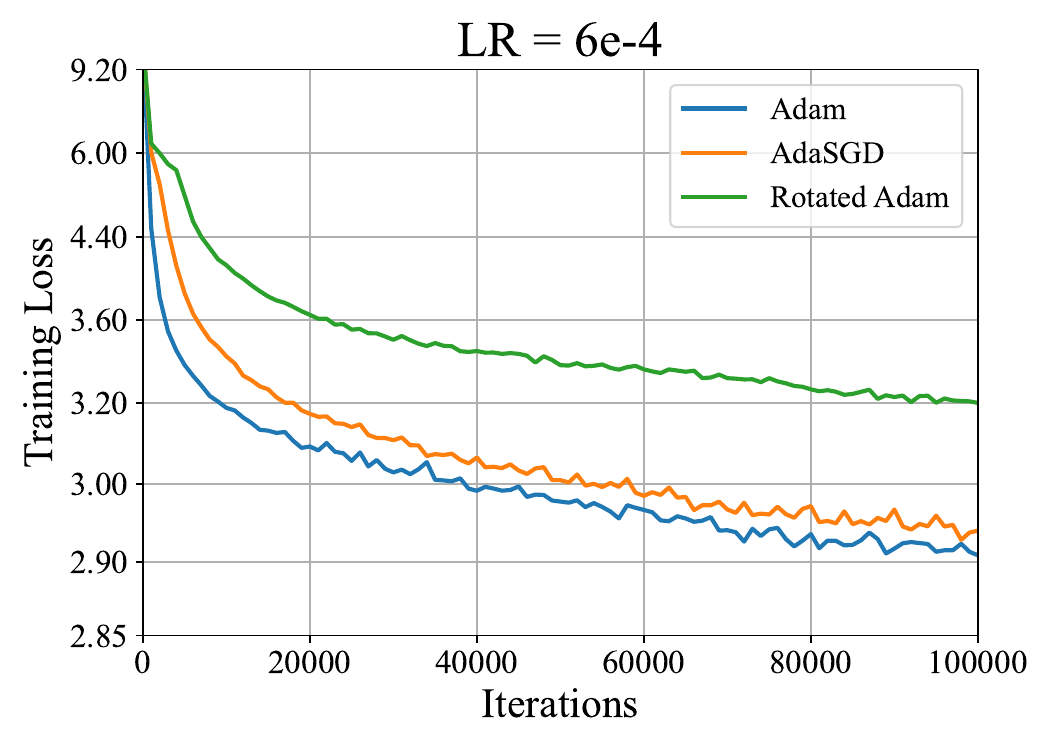}
    \hfill
\includegraphics[width=0.32\textwidth]{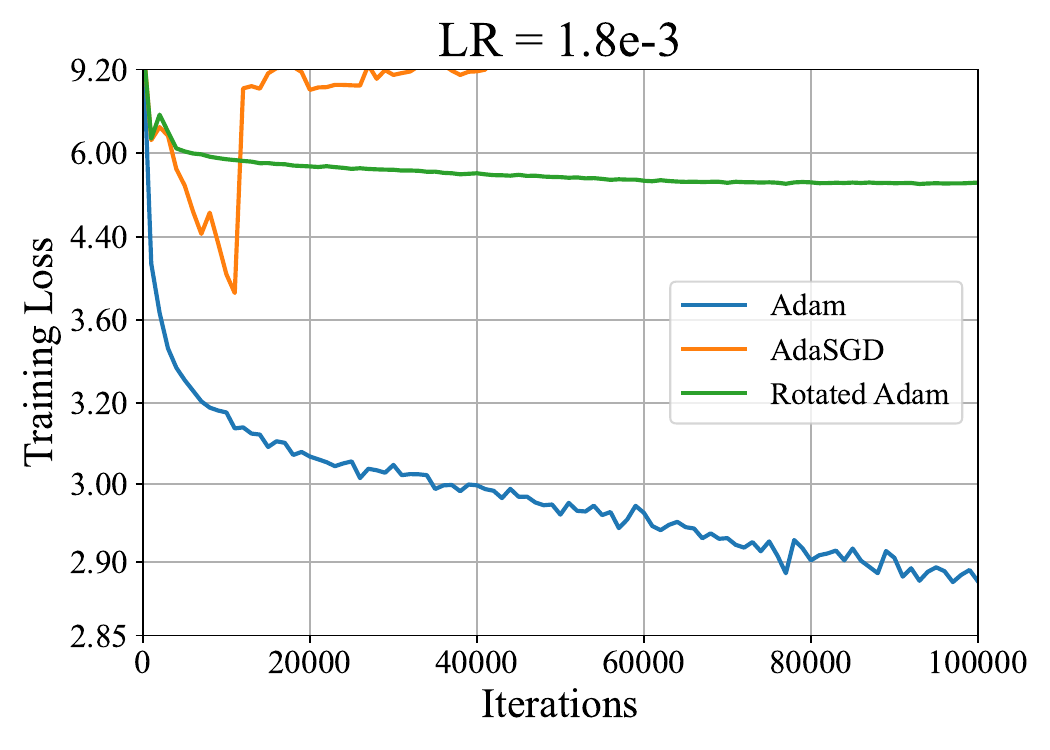}
    \caption{Training losses of \adam, \adasgd and rotated \adam on GPT-2 under different learning rates.  All the optimizers perform worse with smaller learning rate $3 \times 10^{-4}$. Only \adam will perform better with larger learning rate $1.8 \times 10^{-3}$.}
    \label{fig:lr_results}
\end{figure}
GPT-2 small models have more than $100$ million parameters, and thus the size of its hessian of loss as well as the rotation matrix is more than $10^{16}$, which is way more than the storage of the modern computers. Thus our experiments face the following two implementation challenges. Below we briefly discuss them and leave the full details in \Cref{sec:exp_details}.

\paragraph{Efficient generation of random orthogonal matrices.} Due to the huge parameter count, it is computationally infeasible to sample a random orthogonal matrix from Haar measure over the orthogonal group. Thus we decide to sample from an alternative distribution over orthogonal matrices, where the final rotation matrix is the composition of a sequence of simple orthogonal matrices, including random permutation and apply random orthogonal transformation on both sides of matrix-valued parameters.
See details in \Cref{subsec:generating_random_matrix}. Concurrent work~\cite{maes2024understanding} also run \adam on a rotated loss. They use a different way to achieve global rotation efficiently and conduct extensive module-wise rotation experiments for a more fine-grained analysis. 

\paragraph{Efficient estimation of $(1,1)$-norm of Hessian matrix.} The algorithm is defined in \Cref{alg:11norm}. We first subsample a fixed batch of training data for estimating the $(1,1)$-norm of Hessian matrix. The high-level idea is to compute the matrix vector products between Hessian of training loss on this batch and a sequence of random Cauchy vectors. Then we take the $\ell_1$ norm of the coordinate-wise median of the resulting sequence of Hessian vector products. Because the Cauchy distribution is 1-stable, the resulting product is also a vector of Cauchy random variables, and the magnitude of each element equals to $\ell_1$ norm of the corresponding row of the Hessian. Thus with infinitely many samples, the $\ell_1$ norm of the coordinate-wise median converges almost surely to the $(1,1)$-norm of the Hessian. Below we provide a non-asymptotic high-probability multiplicative bound for the estimation error which depends mildly on the dimension $d$. More explanations and the proof of \Cref{thm:measure_11_norm} are in \Cref{subsec:exp_matrix_norm_estimation}. Concurrent work~\cite{maes2024understanding} also proposed another algorithm to measure $(1,1)$-norm. 
\begin{restatable}{theorem}{normestimate}
    \label{thm:measure_11_norm}
    For the estimate in \Cref{alg:11norm} with $n$ Cauchy vectors, it holds that 
    \begin{align*}
      P\left(\abs{\sum_{j=1}^d \textup{median}(\abs{\mathbf{H}_{j,:}}) - \norm{\nabla^2 L(\vx)}_{1,1}}\geq \epsilon \norm{\nabla^2 L(\vx)}_{1,1} \right)\leq 2d \exp(-\frac{n\Delta^2}{2}) +  2\exp{\left(-\frac{2n\epsilon^2 \cos^4((1+\Delta)\frac{\pi}{4})}{\pi^2} \right)}  
    \end{align*}
    for every $\epsilon, \Delta \in (0,1)$ when $n \geq \frac{\pi^3}{2\epsilon^2 \cos^4((1+\Delta)\frac{\pi}{4})}$. 
 
\end{restatable}
In other words, for any $\eps,\delta\in(0,1)$, we can use $n=\Omega(\ln d+ \frac{1}{\epsilon^2}\ln\frac{1}{\delta})$ hessian-vector product of the loss $L$ at parameter $\vx$ and $nd$ extra computation time to get an estimation of (1,1)-norm of $L$ with at most $\epsilon$ multiplicative error and at least probability $1-\delta$.

\subsection{ResNet18 on CIFAR-10}\label{sec:exp_resnet}
\begin{figure}[ht]
    \centering
\includegraphics[width=0.48\textwidth]{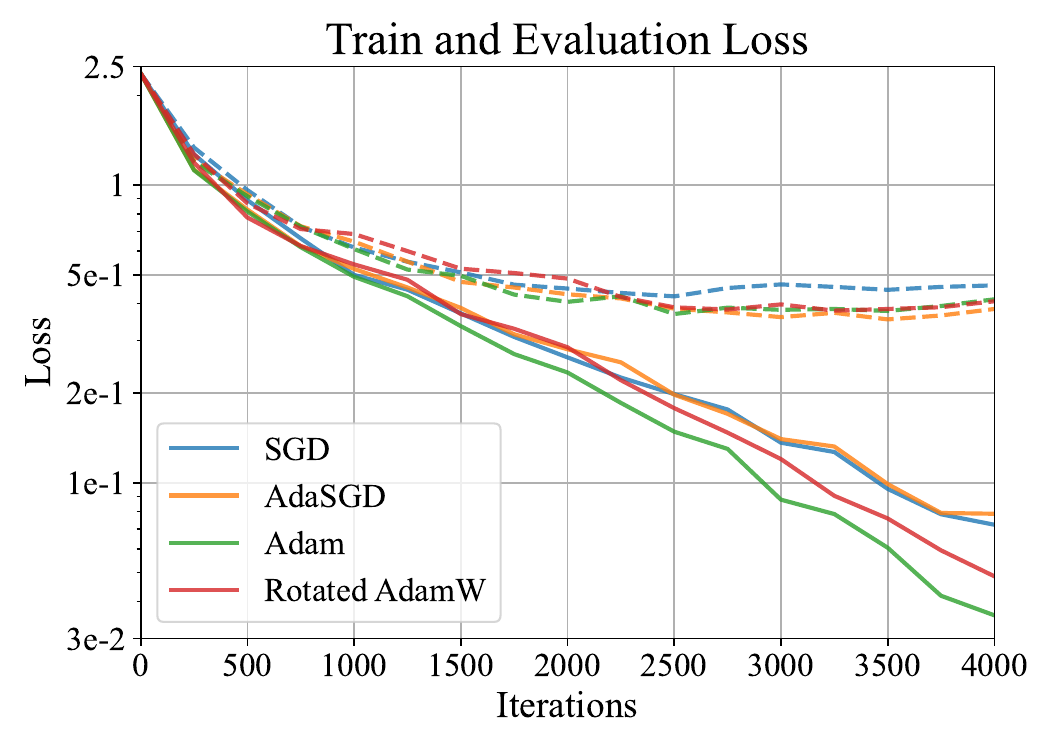}
    % }
    \hfill
\includegraphics[width=0.48\textwidth]{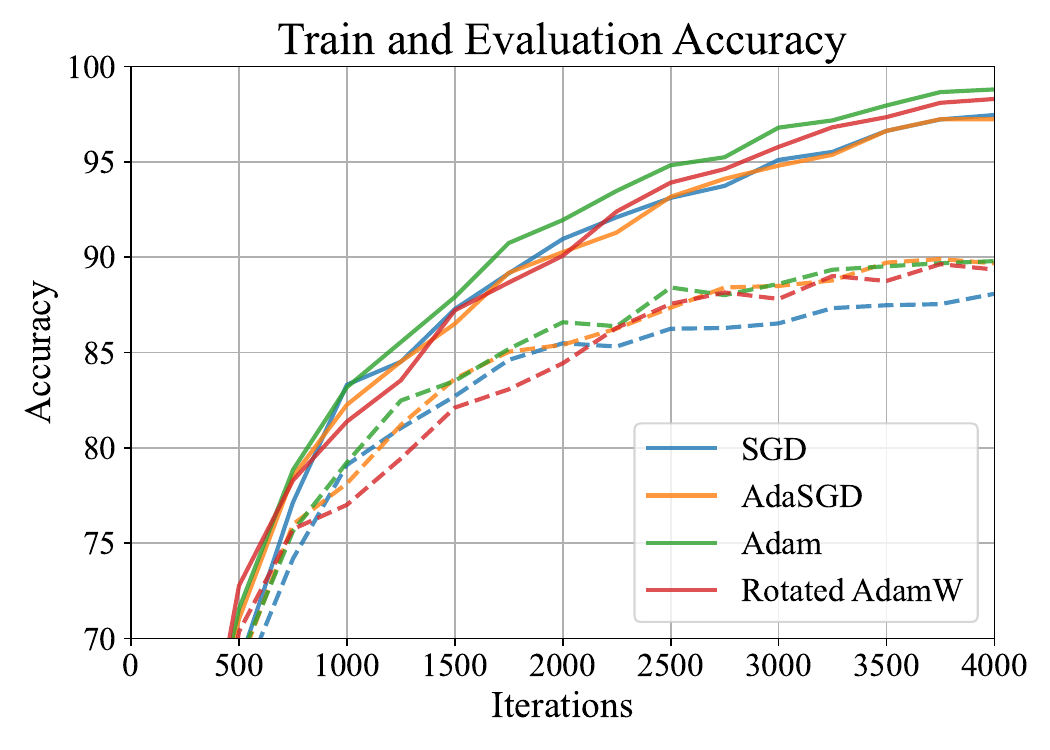}
\vspace{-0.3cm}
    \caption{Training and test losses (left) and accuracy (right) of ResNet18 on CIFAR-10 with \adam, \adasgd, \rotatedadam, and \sgd. We use batch size $256$ and the optimal learning rate in terms of training loss from grid search. Solid and dashed lines correspond to the training and evaluation set metrics respectively. \adam converges faster than other algorithms.}
    \label{fig:resnet_results}
\end{figure}
To further test whether the correlation between $\Phi$-smoothness and the optimization performance holds for architectures other than transformers,
we conduct a similar experiment on ResNet18 trained on CIFAR-10 \citep{Krizhevsky2009learning}. We applied random crop and random horizontal flip augmentations over the training data to promote better generalization. We tuned each optimizer through searching over the same grid of learning rates\footnote{We used the following values: $6.25 \times 10^{-4}$,
$1.25 \times 10^{-3}$,
$2.5 \times 10^{-3}$,
$5.0 \times 10^{-3}$,
$1.0 \times 10^{-2}$,
$2.0 \times 10^{-2}$,
$4.0 \times 10^{-2}$,
$8.0 \times 10^{-2}$,
$1.6 \times 10^{-1}$,
$3.2 \times 10^{-1}$,
$6.4 \times 10^{-1}$,
$1.28 \times 10^{0}$.} The number of iterations is adjusted per batch size to result in ~20 epochs for each training run (for instance, 4000 iterations were used for a batch size of 256, and 1000 iterations were used for a batch size of 1024). For the training loss and training accuracy plotted in \Cref{fig:resnet_results}, it is measured on a subset of augmented training data that is the same size of evaluation set. The evaluation loss and accuracy are measured on the entire evaluation set without the augmentation. Track running stats is set to false at initialization. 

\Cref{fig:resnet_results} depicts the loss and accuracy curves for the best performing hyperparameters chosen over the training set's final loss for batch size $256$.\footnote{We have intentionally limited the number of training iterations to emphasize the difference of optimizers in terms of training speed over generalization.} We also provide the results for other choices of batch size in \Cref{tab:resnet}. 
When it comes to optimization speed, even for ResNet18, \adam is always better than \rotatedadam and they are always better than \adasgd and \sgd across different batch sizes. Note that this does not contradict with common practice of training ResNet with \sgd, where the main goal is to get better generalization and the training budget is large so all optimizers can easily achieve full training accuracy. In our experiment, we study optimization speed and intentionally limit the number of steps.
\begin{table}[ht]
\centering
\begin{tabular}{c| c c c c}
    \toprule
    \textbf{Batch Size} & \textbf{SGD} & \textbf{AdaSGD} & \textbf{Adam} & \textbf{Rotated Adam} \\
    \midrule
    $16$   & $0.0777$ & $0.114$ & $0.064$ & $0.0905$  \\
    $64 $  & $0.0698$ & $0.0854$ & $0.0472$ & $0.0574$ \\
    $256$  & $0.0723$ & $0.0787$ & $0.0359$ & $0.0485$ \\
    $1024$ & $0.1115$ & $0.0915$ & $0.0735$ & $0.0817$ \\
    \bottomrule
\end{tabular}
\caption{Training losses of ResNet for different optimizers and different batch sizes within $20$ epochs. For each setting, we choose the optimal performance over all the learning rates. The 
 performance of \adam is consistently the best among all four optimizers.}\label{tab:resnet}
\end{table}

We also measure the Hessian for checkpoints obtained at batch size $256$ and the results are in \Cref{tab:resnet_norm}. The correlation between norms and convergence rates still holds here. When the $(1,1)$-norm is smaller than $d$ times spectral norm, \adam optimizes faster than \adasgd. 

\begin{table}[h]
\centering
% \vspace{-1.2cm}
\begin{tabular}{c c c c}
    \toprule
    \textbf{Optimizer} & \textbf{Smoothness Metric} & \textbf{Upper Bound} &\textbf{Estimated Value}  \\
    \midrule
    AdaSGD   & $H(L,\Phi_{\mathtt{AdaSGD}})$ &$d \norm{\nabla^2 L(\vx)}_2$& $1.5355$  \\
    Adam  & $H(L,\Phi_{\mathtt{Adam}})$&$\norm{\nabla^2 L(\vx)}_{1,1}$ & $0.0036$  \\
    Rotated Adam  & $H(L\circ \mathcal{R},\Phi_{\mathtt{Adam}})$& $\norm{R^\top\nabla^2 L(\vx) R}_{1,1}$ & $0.9868$  \\
    \bottomrule
\end{tabular}
\caption{Hessian norms for optimal ResNet checkpoints trained with different optimizers and batch size $256$.}\label{tab:resnet_norm}

\end{table}
\section{Related Works}\label{sec:related_works}
\paragraph{Comparison between Adam and SGD}
Previous work tries to analyze the difference between \adam and \sgd from different perspectives. \citet{zhou2018convergence} proves a faster convergence rate of \adam than \sgd when the stochastic gradients are sparse. \citet{zhang2020adaptive} suggests that \sgd suffers more from heavy-tailed noise than \adam. \citet{pan2023toward} claims that \adam has lower directional sharpness because of the effect of coordinate-wise clipping. Other works also consider the coordinate-wise normalization of \adam~\citep{balles2018dissecting,kunstner2023noise}. \citet{kunstner2024heavy} shows that the heavy-tailed class imbalance in language modeling tasks will cause \sgd to converge slower when it can only optimize majority class well. \citet{zhang2024transformers} finds that \adam is better at handling the block heterogeneity of Hessian matrix, which is a specific phenomenon in transformers. 
When viewing \adam as an adaptive method, there are works showing that adaptive methods have an advantage of achieving optimal convergence rate without relying on problem-dependent constant~\citep{ward2020adagrad,levy2021storm+}. \cite{ling2022vectoradam} focuses on the rotation equivariance property in the geometry optimization setting. 
\paragraph{Convergence rate of Adam}
There are many works showing convergence rate for \adam~\citep{zhou2018convergence,chen2018convergence,zou2019sufficient,shi2021rmsprop,guo2021novel,defossez2022simple,zhang2022adam}. Most of them rely on the smoothness of the loss function, which is measured w.r.t. $\ell_2$ norm. \citet{zhang2019gradient} proposes the $(L_0, L_1)$ smoothness condition should be more reasonable than globally bounded smoothness. \citet{li2024convergence} further generalizes the $(L_0, L_1)$ smoothness condition. However, they still focus on the default $\ell_2$ norm which is rotation-invariant. Concurrent work \cite{maes2024understanding} claims analyzing \adam requires rotation dependent assumptions by conducting various rotating experiments. 
To the best of our knowledge, we are the first to assume gradient Lipschitzness under $\ell_\infty$ norm for the analysis on \adam. 

\paragraph{Comparison with \cite{li2024frac}}
\cite{li2024frac} employs the same $\ell_1$ norm for gradient and improves the dependence on dimension $d$ compared to previous results for $\ell_2$ norm. But they still assume the common $\ell_2$ norm smoothness while we adapt their results under $\ell_\infty$ norm smoothness to potentially further improve dependence on $d$. Another drawback of \cite{li2024frac} is setting $\vv_0$ based on noise magnitude $\sigma$. which is impractical in real experiments because $\sigma$ is unknown. Overestimation for $\sigma$ will result in slow convergence because large $\vv_0$ causes \adam to behave similarly with \sgd without adjusting the coordinate-wise learning rate adaptively. In contrast, we allow for general initialization for $\vv_0$ and our convergence rate can work well in both noisy setting and deterministic setting. We also use $1-\beta_2 = \Theta \left( \frac{\log{T}}{T}\right)$ to obtain our convergence rate while \cite{li2024frac} requires $1-\beta_2 = \Theta \left( \frac{1}{T}\right)$. 

\section{Conclusion}\label{sec:conclusion}
We give a new convergence analysis~(\Cref{thm:main}) for \adam in the stochastic non-convex setting using a novel smoothness assumption. We show the convergence rate for the $\ell_1$ norm of the gradient is $O\left(\frac{1}{\sqrt{T}}\right)$ in the deterministic case~(\Cref{cor:deterministic,cor:deterministic_optimal}) and $O\left(\left(\frac{\log T}{T}\right)^{1/4}\right)$ in the stochastic case~(\Cref{cor:stochastic}). We also extend our analysis to blockwise \adam on loss $L$ with respect to an arbitrary partition of the parameters $\Phi$~(\Cref{thm:main_general_norm}) using the corresponding smoothness $H(L,\Phi)$~(\Cref{ass:general_norm_smoothness}). Our bound for \adam involves $(1,1)$-norm of Hessian, rather than the spectral norm of Hessian, which is relevant to the convergence speed of \adasgd. This leads to significantly better smoothness conditions for deep learning models including ResNet-18 and GPT2 empirically. Our experiments also verify that the smoothness measure $H(L,\Phi)$ positively correlates with the optimization speed of blockwise \adam with respect to the partition $\Phi$. 

\section*{Acknowledgement}
The authors would like to thank Khashayar Gatmiry and Sharut Gupta for the helpful discussion and preliminary experiments in exploring the idea of rotated Adam and Xiaoyu Chen for discussions on extending the analysis to blockwise Adam and steepest descent w.r.t. general $\Phi$-norm. ZL is supported by OpenAI Superalignment Grant.

\bibliographystyle{iclr2023_conference}
\bibliography{all}
\appendix
\newpage {
\hypersetup{linkcolor=black}
\tableofcontents
}
\section{Equivariance Property of \texorpdfstring{\adam}{Adam} and \texorpdfstring{\sgd}{SGD}}\label{sec:invariance_property}
\rotation*
\begin{proof}[Proof of \Cref{thm:invariance}]
For \sgd and \adasgd, we will show they are rotation-equivariant by induction. For any rotating transformation $\mathcal{R}(\vx) = \bm{R}\vx$, suppose $\tilde{\vx}_{s} = \mathcal{R}^{-1} (\vx_{s}) = \bm{R}^\top \vx_{s}$ holds for $s\leq t-1$. Then we have that $\tilde{\vg}_t = \nabla_{\tilde{\vx}} \tilde{L}_t(\tilde{\vx}_t) = \bm{R}^\top \nabla_{\vx} L(\bm{R}^{-1} \tvx_{t-1}) = \bm{R}^\top \nabla_{\vx} L(\vx_{t-1}) = \bm{R}^\top \vg_t$ and $\tvm_t = \bm{R}^\top \vm_t$. From the update rule of \sgd, we have that $\tvx_t = \tvx_{t-1} - \eta_t \tvm_t = \bm{R}^\top \vx_{t-1}-\eta_t \bm{R}^\top \vm_t = \bm{R}^\top(\vx_{t-1}-\eta_t \vm_t) = \bm{R}^\top \vx_t$. For the update rule of \adasgd, we further have that $\norm{\tvg_t}_2^2 = \norm{\vg_t}_2^2$ because $\bm{R}$ is an orthogonal matrix. Then $\tv_t = v_t$ and the derivation is similar. 

For \adam and \signgd, it is easy to show by induction they are equivariant w.r.t. any permutating transformation because the operation on gradient is performed on each coordinate separately. We only need to show they are not equivariant w.r.t. a rotating transformation. We choose $\bm{R}=[\frac{1}{\sqrt{2}}, \frac{1}{\sqrt{2}}; \frac{1}{\sqrt{2}}, -\frac{1}{\sqrt{2}}]$, $L_t(\vx) = L(\vx) = 2x_1^2 +x_2^2$. Due to the update rule of \signgd, it can only update $\vx$ and $\tilde{\vx}$ in the direction of $[1,1]$ and $[1,-1]$. But when rotating the update direction on $\tilde{\vx}$ back to the space of $\vx$. The update direction can only be $[1,0]$ or $[0,1]$ that are different from the update direction in the original space. %A comparison of training trajectory can be found in \Cref{fig:toy_example} for better understanding. 
Because the first step in \adam takes the same direction in \signgd, we simultaneously show that both \signgd and \adam are not rotation-equivariant. 
\end{proof}

\section{Convergence rate of \texorpdfstring{\signgd}{SignGD} for deterministic loss}\label{sec:proof_signgd}
\signgdrate*
\begin{proof}[Proof of \Cref{thm:convergence_rate_signgd}]
    We will directly prove a more general verion of \Cref{thm:convergence_rate_signgd}. Because $L$ is $H$-smooth with respect to $\|\cdot\|_\infty$, we have that 
    \begin{align}\label{eq:descent_lemma_nsd}
        L(\vx_{t+1}) - L(\vx_t)  &\leq -\nabla L(\vx_t)^\top (\vx_t-\vx_{t+1}) + \frac{H}{2}\norm{\vx_t-\vx_{t+1}}^2 \nonumber\\
        &\leq -\eta\norm{\nabla L(\vx_t)}_* + \frac{\eta^2 H}{2}\eta^2 
    \end{align}
    This implies that 
    \begin{align*}
        \min_{1 \leq t \leq T}\norm{\nabla L(\vx_t)}_* \leq \frac{1}{T}\sum_{t=1}^T \norm{\nabla L(\vx_t)}_* \le \frac{L(\vx_0) - L(\vx_T)}{T\eta} + \frac{H\eta}{2},
    \end{align*}
    which completes the proof.
\end{proof}

% \section{Proof Details}
\section{Proof for Convergence Rate of Blockwise \texorpdfstring{\adam}{Adam}}\label{sec:proof_detail}
As mentioned in \Cref{sec:general_theory}, we will use \Cref{lem:second_order_general_norm} to show the relationship between blockwisely-smoothness w.r.t. $\Phi$-norm and $H(L,\Phi)$. 
\secondorder*
\begin{proof}[Proof of \Cref{lem:second_order_general_norm}]
    From \Cref{ass:general_norm_smoothness}, we know that 
    \begin{align*}
       H_b &\ge \sup_{\vx, \vDelta} \frac{ \norm{\nabla_{(b)} L(\vx+\vDelta) - \nabla_{(b)} L(\vx)}_2}{\sqrt{d_b}\max_{b' \in [B]} \frac{\norm{\vDelta_{(b')}}_2}{\sqrt{d_{b'}}}} \\
       &=\sup_{\vx, \vDelta} \frac{ \norm{\nabla_{(b), :}^2 L(\vx) \vDelta}_2}{\sqrt{d_b}\max_{b' \in [B]} \frac{\norm{\vDelta_{(b')}}_2}{\sqrt{d_{b'}}}}\\
       &=\sup_{\vx, \vDelta} \frac{ \norm{\sum_{b' =1}^B\nabla_{(b), (b')}^2 L(\vx) \vDelta_{(b')}}_2}{\sqrt{d_b}\max_{b' \in [B]} \frac{\norm{\vDelta_{(b')}}_2}{\sqrt{d_{b'}}}}\\
       &=\sup_{\vx, \vDelta, \norm{\vDelta'_{(b)}}_2 \leq 1} \frac{ \inner{\vDelta'_{(b)}}{\sum_{b' =1}^B\nabla_{(b), (b')}^2 L(\vx) \vDelta_{(b')}}}{\sqrt{d_b}\max_{b' \in [B]} \frac{\norm{\vDelta_{(b')}}_2}{\sqrt{d_{b'}}}}\\
       &=\sup_{\vx, \norm{\vDelta_{(b')}}_2 \leq \sqrt{d_{b'}}, \norm{\vDelta'_{(b)}}_2 \leq 1}\frac{1}{\sqrt{d_b}} \inner{\vDelta'_{(b)}}{\sum_{b' =1}^B\nabla_{(b), (b')}^2 L(\vx) \vDelta_{(b')}}\\
       &=\sup_{\vx, \vDelta, \vDelta'}\sum_{b' =1}^B \frac{ \sqrt{d_{b'}}}{\sqrt{d_b}\norm{\vDelta'_{(b)}}_2 \norm{\vDelta_{(b')}}_2} \inner{\vDelta'_{(b)}}{\nabla_{(b), (b')}^2 L(\vx) \vDelta_{(b')}}.
    \end{align*}
Then for any $\vx$ and $\vDelta$, we know that 
\begin{align*}
    H_b \norm{\vDelta_{(b)}}_2^2 &\geq \norm{\vDelta_{(b)}}_2^2 \sum_{b'=1}^B \frac{ \sqrt{d_{b'}}}{\sqrt{d_b}\norm{\vDelta_{(b)}}_2 \norm{\vDelta_{(b')}}_2} \abs{\inner{\vDelta_{(b)}}{\nabla_{(b), (b')}^2 L(\vx) \vDelta_{(b')}}}\\
    &= \sum_{b'=1}^B \frac{\sqrt{d_{b'}}\norm{\vDelta_{(b)}}_2}{\sqrt{d_b} \norm{\vDelta_{(b')}}_2} \abs{\inner{\vDelta_{(b)}}{\nabla_{(b), (b')}^2 L(\vx) \vDelta_{(b')}}}
\end{align*}
and 
\begin{align*}
    &2\sum_{b=1}^B H_b \norm{\vDelta_{(b)}}_2^2\\
    =& \sum_{b=1}^B H_b \norm{\vDelta_{(b)}}_2^2 + \sum_{b'=1}^B H_{b'} \norm{\vDelta_{(b')}}_2^2\\
    \geq& \sum_{b=1}^B\sum_{b'=1}^B \frac{\sqrt{d_{b'}}\norm{\vDelta_{(b)}}_2}{\sqrt{d_b} \norm{\vDelta_{(b')}}_2} \abs{\inner{\vDelta_{(b)}}{\nabla_{(b), (b')}^2 L(\vx) \vDelta_{(b')}}}
    + \sum_{b'=1}^B\sum_{b=1}^B \frac{\sqrt{d_{b}}\norm{\vDelta_{(b')}}_2}{\sqrt{d_{b'}} \norm{\vDelta_{(b)}}_2} \abs{\inner{\vDelta_{(b')}}{\nabla_{(b'), (b)}^2 L(\vx) \vDelta_{(b)}}}\\
    \geq&2 \sum_{b=1}^B \sum_{b'=1}^B \abs{\vDelta_{(b)}^\top \nabla_{(b), (b')}^2 L(\vx) \vDelta_{(b')}} \geq 2 \abs{\vDelta^\top \nabla^2 L(\vx) \vDelta}. 
\end{align*}
% The last inequality comes from mean inequality.  
\end{proof}
We will use \Cref{lem:momentum_ratio} to better control the growth of the sum of second order term. 
\momentumratio*
\begin{proof}[Proof of \Cref{lem:momentum_ratio}]
Notice that $1-x \leq \ln{\frac{1}{x}}$ for any positive $x$. We can have that
\begin{align}\label{eq:update_square_upper_bound}
        \sum_{t=1}^{T} \frac{g_t^2}{v_t}&\leq\sum_{t=1}^{T} \frac{v_t-\beta_2 v_{t-1}}{(1-\beta_2) v_t}\notag\\
        &=\sum_{t=1}^T \left[1 + \frac{\beta_2}{1-\beta_2} \left( 1-\frac{v_{t-1}}{v_t} \right) \right]\notag\\
        &\leq T + \frac{\beta_2}{1-\beta_2} \sum_{t=1}^T \ln{\frac{v_{t}}{v_{t-1}}}
        \notag\\
        &=T+\frac{\beta_2}{1-\beta_2} \ln{\frac{v_T}{v_0}}.
    \end{align}
    when $v_0 \neq 0$. 
When $v_0=0$, we can still have that 
\begin{align*}
    \sum_{t=1}^T \frac{g_t^2}{v_t} &\leq \frac{1}{1-\beta_2} + \sum_{t=2}^T \frac{g_t^2}{v_t} \\
    &\leq \frac{1}{1-\beta_2} + (T-1) + \frac{\beta_2}{1-\beta_2} \ln{\frac{v_T}{v_1}}\\
    &=T + \frac{\beta_2}{1-\beta_2}\ln{\frac{v_T}{v_1/e}}. 
\end{align*}
\end{proof}

Next we deal with the first order term by approximating it with a deterministic term. Recall the notation defined in \Cref{sec:proof_sketch}. $\vg_t$ denotes the gradient of mini-batch $L_t(\vx_{t-1})$ at step $t$. And $\E\left[\vg_t \middle|\vx_{t-1}\right]=\nabla L(\vx_{t-1})$ because $\E{L_t}=L$. The full-batch gradient is $\bar{\vg}_{t} = \nabla L(\vx_{t-1})$. Different kinds of second-order momentum are defined in the following way
\begin{align*}
    v_{t,b} &= \beta_2^t \norm{\vg_{1,(b)}}_2^2 /d_b+\left(1-\beta_2\right)\sum_{j=0}^{t-1} \beta_2^{j} \left(\norm{\vg_{t-j, (b)}}_2^2\right)/d_b,\\
    \Tilde{v}_{t,b} &= (1-\beta_2) \left(\norm{\bar{\vg}_{t,(b)}}_2^2/d_b+\sigma_b^2\right) + \beta_2 v_{t-1,b}.
\end{align*}

\firstorder*
\begin{proof}[Proof of \Cref{lem:first_order_approx_rmsprop}]
The first order change in block $b$ can decomposed into two terms. 
\begin{equation}\label{eq:first_order_decomposition}
   \begin{aligned}
    \E \sum_{t=1}^T \sum_{\Phi(i)=b} \frac{g_{t,i} \bar{g}_{t,i}}{\sqrt{v_{t,b}+\epsilon}}&= \E \sum_{t=1}^T \sum_{\Phi(i)=b} \frac{g_{t,i}\bar{g}_{t,i}}{\sqrt{\tilde{v}_{t,b}+\epsilon}} + \E \left[\sum_{t=1}^T \sum_{\Phi(i)=b} \frac{g_{t,i} \bar{g}_{t,i}}{\sqrt{v_{t,b}+\epsilon}} -\frac{g_{t,i}\bar{g}_{t,i}}{\sqrt{\tilde{v}_{t,b}+\epsilon}}\right]\\
    &=\E \sum_{t=1}^T \sum_{\Phi(i)=b} \E\left[ \frac{g_{t,i} \bar{g}_{t,i}}{\sqrt{\tilde{v}_{t,b}+\epsilon}}\middle|\vx_{t-1} \right]+ \E \left[\sum_{t=1}^T\sum_{\Phi(i)=b}  \frac{g_{t,i} \bar{g}_{t,i}}{\sqrt{v_{t,b}+\epsilon}} -\frac{g_{t,i} \bar{g}_{t,i}}{\sqrt{\tilde{v}_{t,b}+\epsilon}}\right]\\
    &=\E \sum_{t=1}^T \sum_{\Phi(i)=b} \frac{\bar{g}_{t,i}^2}{\sqrt{\tilde{v}_{t,b}+\epsilon}} + \E \left[\sum_{t=1}^T \sum_{\Phi(i)=b} \frac{g_{t,i} \bar{g}_{t,i}}{\sqrt{v_{t,b}+\epsilon}} -\frac{g_{t,i}\bar{g}_{t,i}}{\sqrt{\tilde{v}_{t,b}+\epsilon}}\right]\\
\end{aligned} 
\end{equation}
For the second term, we have that 
\begin{align*}
    &\sum_{\Phi(i)=b} \abs{g_{t,i} \bar{g}_{t,i} \left(\frac{1}{\sqrt{v_{t,b}+\epsilon}}-\frac{1}{\sqrt{\tilde{v}_{t,b}+\epsilon}}\right)} \\=&\sum_{\Phi(i)=b} \frac{\abs{g_{t,i} \bar{g}_{t,i} \left( \tilde{v}_{t,b}-v_{t,b} \right)}}{\sqrt{v_{t,b}+\epsilon}\sqrt{\tilde{v}_{t,b}+\epsilon }\left(\sqrt{v_{t,b}+\epsilon} + \sqrt{\tilde{v}_{t,b}+\epsilon} \right)}\\
    =&\sum_{\Phi(i)=b} \frac{\abs{g_{t,i} \bar{g}_{t,i} (1-\beta_2)\left(\norm{\bar{\vg}_{t,(b)}}_2^2 /d_b +\sigma_b^2-\norm{\vg_{t,(b)}}_2^2 /d_b\right)}}{\sqrt{v_{t,b}+\epsilon}\sqrt{\tilde{v}_{t,b}+\epsilon }\left(\sqrt{v_{t,b}+\epsilon} + \sqrt{\tilde{v}_{t,b}+\epsilon} \right)}\\
    =&\sum_{\Phi(i)=b} \frac{\abs{g_{t,i} \bar{g}_{t,i} (1-\beta_2)\left( \sqrt{\norm{\bar{\vg}_{t,(b)}}_2^2 /d_b +\sigma_b^2}+\sqrt{\norm{\vg_{t,(b)}}_2^2/d_b} \right)\left( \sqrt{\norm{\bar{\vg}_{t,(b)}}_2^2/d_b +\sigma_b^2}-\sqrt{\norm{\vg_{t,(b)}}_2^2/d_b} \right)}}{\sqrt{v_{t,b}+\epsilon}\sqrt{\tilde{v}_{t,b} +\epsilon}\left(\sqrt{v_{t,b}+\epsilon} + \sqrt{\tilde{v}_{t,b}+\epsilon} \right)}\\
    \leq&\sum_{\Phi(i)=b} \frac{\abs{g_{t,i} \bar{g}_{t,i} \sqrt{1-\beta_2}\left( \sqrt{\norm{\bar{\vg}_{t,(b)}}_2^2/d_b +\sigma_b^2}-\sqrt{\norm{\vg_{t,(b)}}_2^2/d_b} \right)}}{\sqrt{v_{t,b}+\epsilon}\sqrt{\tilde{v}_{t,b} +\epsilon}}\\
    \leq& \frac{1}{2}\sum_{\Phi(i)=b} \frac{\bar{g}_{t,i}^2 }{\sqrt{\tilde{v}_{t,b}+\epsilon}} \frac{\left( \sqrt{\norm{\bar{\vg}_{t,(b)}}_2^2/d_b +\sigma_b^2}-\sqrt{\norm{\vg_{t,(b)}}_2^2/d_b} \right)^2}{\E [\left( \sqrt{\norm{\bar{\vg}_{t,(b)}}_2^2/d_b +\sigma_b^2}-\sqrt{\norm{\vg_{t,(b)}}_2^2/d_b} \right)^2 \vert \vx_{t-1}]} \\
    +& \frac{1}{2}\sum_{\Phi(i)=b} \frac{(1-\beta_2) g_{t,i}^2 \E [\left( \sqrt{\norm{\bar{\vg}_{t,(b)}}_2^2/d_b +\sigma_b^2}-\sqrt{\norm{\vg_{t,(b)}}_2^2/d_b} \right)^2 \vert \vx_{t-1}] }{(v_{t,b}+\epsilon) \sqrt{\tilde{v}_{t,b}+\epsilon} } 
\end{align*}
The first inequality is because $v_{t,b}+\epsilon \geq (1-\beta_2) \norm{\vg_{t,(b)}}_2^2/d_b$ and $\tilde{v}_{t,b}+\epsilon \geq (1-\beta_2) \left(\norm{\bar{\vg}_{t,(b)}}_2^2/d_b + \sigma_b^2 \right)$. 
For the first term, it will be exactly $\frac{1}{2} \frac{\norm{\bar{\vg}_{t,(b)}}_2^2}{\sqrt{\tilde{v}_{t,b}+\epsilon}}$ after taking expectation conditional on $\vx_{t-1}$. For the second term, we have the following inequality
\begin{align*}
    &\E \left[\left(\sqrt{\norm{\bar{\vg}_{t,(b)}}_2^2/d_b +\sigma_b^2}-\sqrt{\norm{\vg_{t,(b)}}_2^2/d_b} \right)^2 \middle| \vx_{t-1} \right] \\
    =& \E \left[\norm{\bar{\vg}_{t,(b)}}_2^2/d_b + \sigma_b^2 + \sum_{\Phi(j)=b}g_{t,j}^2/d_b -2 \sqrt{\norm{\vg_{t,(b)}}_2^2/d_b} \sqrt{ \norm{\bar{\vg}_{t,(b)}}_2^2 /d_b+\sigma_i^2} \middle| \vx_{t-1} \right]\\
    \leq& 2\left(\norm{\bar{\vg}_{t,(b)}}_2^2/d_b + \sigma_b^2\right) -2 \sqrt{\norm{\bar{\vg}_{t,(b)}}_2^2/d_b + \sigma_b^2} \E \left[ \sqrt{\norm{\vg_{t,(b)}}_2^2/d_b}\middle| \vx_{t-1} \right] \\
    \leq& 2\left(\norm{\bar{\vg}_{t,(b)}}_2^2/d_b + \sigma_b^2\right) -2 \sqrt{\norm{\bar{\vg}_{t,(b)}}_2^2/d_b + \sigma_b^2} \sqrt{\norm{\bar{\vg}_{t,(b)}}_2^2/d_b}\\
    =&2\sqrt{\norm{\bar{\vg}_{t,(b)}}_2^2/d_b + \sigma_b^2} \left(\sqrt{\norm{\bar{\vg}_{t,(b)}}_2^2/d_b + \sigma_b^2} -\sqrt{\norm{\bar{\vg}_{t,(b)}}_2^2/d_b } \right) \\
    \leq& 2\sqrt{\norm{\bar{\vg}_{t,(b)}}_2^2/d_b + \sigma_b^2} \sigma_b.
\end{align*}
The first inequality comes from \Cref{ass:bounded_noise}. The second inequality is because $\ell_2$ norm is a convex function. 
Then we know that 
\begin{align*}
    &\sum_{\Phi(i)=b} \frac{(1-\beta_2) g_{t,i}^2 \E \left[\left( \sqrt{\norm{\bar{\vg}_{t,(b)}}_2^2/d_b +\sigma_b^2}-\sqrt{\norm{\vg_{t,(b)}}_2^2/d_b} \right)^2 \vert \vx_{t-1} \right] }{(v_{t,b}+\epsilon) \sqrt{\tilde{v}_{t,b}+\epsilon} } \\
    \leq&\sum_{\Phi(i)=b} \frac{(1-\beta_2) g_{t,i}^2 2\sqrt{\norm{\bar{\vg}_{t,(b)}}_2^2/d_b + \sigma_b^2} \sigma_b}{(v_{t,b}+\epsilon) \sqrt{\tilde{v}_{t,b}+\epsilon} } \\
    \leq& 2\sqrt{1-\beta_2} \sigma_b\sum_{\Phi(i)=b}\frac{g_{t,i}^2}{v_{t,b}+\epsilon}.
\end{align*}
Then back to \Cref{eq:first_order_decomposition}, we have that 
\begin{align*}
 \E \sum_{t=1}^T \sum_{\Phi(i)=b} \frac{g_{t,i} \bar{g}_{t,i}}{\sqrt{v_{t,b}+\epsilon}} &=\E \sum_{t=1}^T \sum_{\Phi(i)=b} \frac{\bar{g}_{t,i}^2}{\sqrt{\tilde{v}_{t,b}+\epsilon}} + \E \left[\sum_{t=1}^T \sum_{\Phi(i)=b} \frac{g_{t,i} \bar{g}_{t,i}}{\sqrt{v_{t,b}+\epsilon}} -\frac{g_{t,i}\bar{g}_{t,i}}{\sqrt{\tilde{v}_{t,b}+\epsilon}}\right]\\
 &\geq \E \sum_{t=1}^T \sum_{\Phi(i)=b} \frac{\bar{g}_{t,i}^2}{\sqrt{\tilde{v}_{t,b}+\epsilon}} - \frac{1}{2}\E \sum_{t=1}^T \sum_{\Phi(i)=b} \frac{\bar{g}_{t,i}^2}{\sqrt{\tilde{v}_{t,b}+\epsilon}}-\frac{1}{2} 2\sqrt{1-\beta_2}\sigma_b \E \sum_{t=1}^T \frac{\norm{\vg_{t,(b)}}_2^2}{v_{t,b}+\epsilon}\\
 &=\frac{1}{2}\E \sum_{t=1}^T \sum_{\Phi(i)=b} \frac{\bar{g}_{t,i}^2}{\sqrt{\tilde{v}_{t,b}+\epsilon}}-\sqrt{1-\beta_2}\sigma_b \E \sum_{t=1}^T \frac{\norm{\vg_{t,(b)}}_2^2}{v_{t,b}+\epsilon}\\
\end{align*}
For the second term, we can apply \Cref{lem:momentum_ratio} and get that 
\begin{equation*}
    \sum_{t=1}^T \frac{\sum_{\Phi(i)=b} g_{t,i}^2/d_b}{v_{t,b}+\epsilon} \leq T + \frac{\beta_2}{1-\beta_2} \ln{\frac{v_{T,b}+\epsilon}{v_{0,b}+\epsilon}}.
\end{equation*}
Combining these two terms, we can get that 
\begin{align*}
\E \sum_{t=1}^T \sum_{\Phi(i)=b}\frac{g_{t,i} \bar{g}_{t,i}}{\sqrt{v_{t,b}+\epsilon}} &\geq \frac{1}{2} \E \sum_{t=1}^T \sum_{\Phi(i)=b}\frac{\bar{g}_{t,i}^2}{\sqrt{\tilde{v}_{t,b}+\epsilon}}
-\sqrt{1-\beta_2}T d_b \sigma_b  - \frac{d_b \sigma_b \beta_2}{\sqrt{1-\beta_2}} \E \left[\ln{\frac{v_{T,b}+\epsilon}{v_{0,b}+\epsilon}} \right]. 
\end{align*}
\end{proof}
Next we need \Cref{lem:denominator} to deal with the denominator in the approximated first order term. The lemma is largely inspired by Lemma 6 in \citet{li2024frac}, where we further generalize it to the case of block-wise \adam. 

\denominator*
\begin{proof}[Proof of \Cref{lem:denominator}]
For each $t \leq T$, we have that
    \begin{align*}
       &\E \left[\sqrt{\tilde{v}_{t,b}+\epsilon} \right]\\
       =& \E \left[\sqrt{\beta_2 v_{t-1,b} + (1-\beta_2)(\norm{\bar{\vg}_{t,(b)}}_2^2/d_b + \sigma_b^2)+\epsilon} \right] \\
       =& \E \left[ \frac{\beta_2 v_{t-1,b}+(1-\beta_2) \sigma_b^2 + \epsilon}{\sqrt{\beta_2 v_{t-1,b} + (1-\beta_2)(\sum_{\Phi(i)=b}\bar{g}_{t,i}^2/d_b + \sigma_b^2)+\epsilon}} \right] + (1-\beta_2) \E \left[\frac{\norm{\bar{\vg}_{t,(b)}}_2^2/d_b}{\sqrt{\tilde{v}_{t,b}+\epsilon}}  \right]\\
       \leq & \E \left[ \sqrt{\beta_2 v_{t-1,b}+(1-\beta_2) \sigma_b^2 + \epsilon} \right] + (1-\beta_2) \E \left[\frac{\norm{\bar{\vg}_{t,(b)}}_2^2/d_b}{\sqrt{\tilde{v}_{t,b}+\epsilon}} \right]. 
    \end{align*}
And for each $s \leq t-1$, we have that 
\begin{align*}
    &\E \left[\sqrt{\beta_2^s v_{t-s,b} + (1-\beta_2^s) \sigma_b^2 +\epsilon} \right]\\
    =& \E \left[\sqrt{\beta_2^{s+1} v_{t-s-1,b} + \beta_2^s (1-\beta_2) \sum_{\Phi(i)=b}g_{t-s,i}^2/d_b + (1-\beta_2^s) \sigma_b^2 +\epsilon} \right]\\
    =& \E \left[ \E \left[\sqrt{\beta_2^{s+1} v_{t-s-1,b} + \beta_2^s (1-\beta_2) \sum_{\Phi(i)=b} g_{t-s,i}^2/d_b + (1-\beta_2^s) \sigma_b^2 +\epsilon}\middle| \vx_{t-s-1}\right]\right]\\
    \leq & \E \left[ \sqrt{\beta_2^{s+1} v_{t-s-1,b} + \beta_2^s (1-\beta_2) \E\left[ \sum_{\Phi(i)=b} g_{t-s,i}^2/d_b \middle| \vx_{t-s-1} \right] + (1-\beta_2^s) \sigma_b^2 +\epsilon} \right]\\
    \leq & \E \left[ \sqrt{\beta_2^{s+1} v_{t-s-1,b} + \beta_2^s (1-\beta_2) \sum_{\Phi(i)=b} \bar{g}_{t-s,i}^2/d_b + (1-\beta_2^{s+1}) \sigma_b^2 +\epsilon} \right]\\
    = & \E \left[ \frac{\beta_2^{s+1} v_{t-s-1,b} + (1-\beta_2^{s+1})\sigma_b^2 + \epsilon}{\sqrt{\beta_2^{s+1} v_{t-s-1,b} + \beta_2^s (1-\beta_2) \sum_{\Phi(i)=b} \bar{g}_{t-s,i}^2/d_b + (1-\beta_2^{s+1}) \sigma_b^2 +\epsilon}}\right]\\
    + & \E \left[ \frac{\beta_2^s (1-\beta_2) \sum_{\Phi(i)=b} \bar{g}_{t-s,i}^2/d_b}{\sqrt{\beta_2^{s+1} v_{t-s-1,b} + \beta_2^s (1-\beta_2) \sum_{\Phi(i)=b} \bar{g}_{t-s,i}^2/d_b + (1-\beta_2^{s+1}) \sigma_b^2 +\epsilon}}\right]\\
    \leq & \E \left[ \sqrt{\beta_2^{s+1} v_{t-s-1,b} + (1-\beta_2^{s+1}) \sigma_b^2 +\epsilon} \right] + \sqrt{\beta_2^s}(1-\beta_2) \E \left[\frac{\sum_{\Phi(i)=b} \bar{g}_{t-s,i}^2/d_b}{\sqrt{\tilde{v}_{t-s,b}+\epsilon}} \right].
\end{align*}
By summing the above inequality over $s=1, \cdots, t-1$, we have that
\begin{align*}
    &\E \left[ \sqrt{\beta_2 v_{t-1,b}+(1-\beta_2) \sigma_b^2 + \epsilon} \right] \\\leq& \E \left[\sqrt{\beta_2^t v_{0,b} + (1-\beta_2^t) \sigma_b^2 +\epsilon} \right] + \sum_{s=1}^{t-1} \sqrt{\beta_2^s} (1-\beta_2) \E \left[\frac{\sum_{\Phi(i)=b} \bar{g}_{t-s,i}^2/d_b}{\sqrt{\tilde{v}_{t-s,b}+\epsilon}} \right]\\
    \leq & \sqrt{\beta_2^t v_{0,b}} + \sqrt{\sigma_b^2 + \epsilon} + \sum_{s=1}^{t-1} \sqrt{\beta_2^s} (1-\beta_2) \E \left[\frac{\sum_{\Phi(i)=b} \bar{g}_{t-s,i}^2/d_b}{\sqrt{\tilde{v}_{t-s,b}+\epsilon}} \right].
\end{align*}
and 
\begin{align*}
    \E \left[\sqrt{\tilde{v}_{t,b}+\epsilon} \right] \leq \sqrt{\beta_2^t v_{0,b}} + \sqrt{\sigma_b^2+\epsilon} + \sum_{s=0}^{t-1} \sqrt{\beta_2^s}(1-\beta_2) \E \left[\frac{\sum_{\Phi(i)=b} \bar{g}_{t-s,i}^2/d_b}{\sqrt{\tilde{v}_{t-s,b}+\epsilon}} \right].
\end{align*}
By summing the above inequality over $t=\frac{T}{2}+1, \cdots, T$, we have that 
\begin{align*}
    \sum_{t=\frac{T}{2}+1}^T \left[\sqrt{\tilde{v}_{t,b}+\epsilon} \right] &\leq \sum_{t=\frac{T}{2}+1}^T \sqrt{\beta_2^t v_{0,b}} + \frac{T}{2} \sqrt{\sigma_b^2 + \epsilon} + \sum_{t=\frac{T}{2}+1}^T \sum_{s=0}^{t-1} \sqrt{\beta_2^s}(1-\beta_2) \E \left[\frac{\sum_{\Phi(i)=b} \bar{g}_{t-s,i}^2/d_b}{\sqrt{\tilde{v}_{t-s,b}+\epsilon}} \right] \\
    &\leq \frac{\beta_2^\frac{T}{4}}{1-\sqrt{\beta_2}} \sqrt{v_{0,b}} + \frac{T}{2} \sqrt{\sigma_b^2 + \epsilon} + \frac{1-\beta_2}{1-\sqrt{\beta_2}} \sum_{t=1}^T \E \left[\frac{\norm{\bar{\vg}_{t,(b)}}_2^2/d_b}{\sqrt{\tilde{v}_{t,b}+\epsilon}} \right]\\
    &= \frac{\beta_2^\frac{T}{4}}{1-\sqrt{\beta_2}} \sqrt{v_{0,b}} + \frac{T}{2} \sqrt{\sigma_b^2 + \epsilon} + (1+\sqrt{\beta_2}) \sum_{t=1}^T \E \left[\frac{\norm{\bar{\vg}_{t,(b)}}_2^2/d_b}{\sqrt{\tilde{v}_{t,b}+\epsilon}} \right]\\
    &\leq \frac{2\beta_2^\frac{T}{4}}{1-\beta_2} \sqrt{v_{0,b}} + \frac{T}{2} \sigma_b + \frac{T}{2} \sqrt{\epsilon} + 2 \sum_{t=1}^T \E \left[\frac{\norm{\bar{\vg}_{t,(b)}}_2^2/d_b}{\sqrt{\tilde{v}_{t,b}+\epsilon}} \right]. 
\end{align*}
\end{proof}

This following \Cref{lem:vt_growth} is to control the growth of $v_{T,b}$ so that the right hand side in \Cref{lem:momentum_ratio} is indeed $\Theta \left( T+\frac{\log{T}}{1-\beta_2} \right)$ instead of $\Theta(\frac{T}{1-\beta_2})$ when all the constants are $\text{poly}(T)$.
\begin{lemma}\label{lem:vt_growth}
For any $T$, it holds that
\begin{align*}
    \ln\frac{\E \max_{b\in[B]} v_{T, b}+\epsilon}{v_0+\epsilon} &\leq 2\ln\left(1+ \frac{ \sum_{b=1}^B \sigma_b^2 + \norm{\nabla L(\vx_0)}_{\Phi}^2 + \sum_{b\in[B]} H_b^2 d_b \eta^2 T(T + \frac{1}{1-\beta_2})} {v_0+\epsilon}\right) + \ln{32}
\end{align*}
\end{lemma}
\begin{proof}[Proof of \Cref{lem:vt_growth}]
From the definition of $v_{t,b}$ and \Cref{ass:general_norm_noise}, we have that
    \begin{align*}
    &\E \max_{b\in[B]} v_{t,b}\\
    =&\E \max_{b\in[B]} \left[\beta_2^t v_{0,b} + (1-\beta_2)\sum_{s=1}^t \beta_2^{t-s} \norm{\vg_{s,(b)}}_2^2/d_b \right]\\
    \leq& \beta_2^t \norm{\vv_0}_\infty + (1-\beta_2) \E \max_{b\in[B]} \sum_{s=1}^t \beta_2^{t-s} \norm{\vg_{s,(b)}}_2^2/d_b \\
    =& \beta_2^t \norm{\vv_0}_\infty + (1-\beta_2) \E \max_{b\in[B]} \sum_{s=1}^t \beta_2^{t-s} \norm{\E [\vg_{s,(b)}| \vx_{s-1}] + \vg_{s,(b)}-\E [\vg_{s,(b)}| \vx_{s-1}]}_2^2/d_b \\
    \leq& \beta_2^t \norm{\vv_0}_\infty + (1-\beta_2) \E \max_{b\in[B]} \sum_{s=1}^t \beta_2^{t-s} \left[2\norm{\E [\vg_{s,(b)}| \vx_{s-1}]}_2^2 +  2\norm{\vg_{s,(b)}-\E [\vg_{s,(b)}| \vx_{s-1}]}_2^2\right] /d_b\\
    \leq& \beta_2^t \norm{\vv_0}_\infty+2(1-\beta_2) \E \sum_{b=1}^B \sum_{s=1}^t \beta_2^{t-s} \norm{\vg_{s,(b)}-\E [\vg_{s,(b)}| \vx_{s-1}]}_2^2/d_b + 2(1-\beta_2) \E \max_{b \in [B]} \sum_{s=1}^t \beta_2^{t-s} \norm{\nabla_{(b)} L(\vx_{s-1})}_2^2/d_b \\
    \leq& \beta_2^t \norm{\vv_0}_\infty +2 (1-\beta_2^t) \sum_{b=1}^B \sigma_b^2 + 2(1-\beta_2) \E \max_{b \in [B]} \sum_{s=1}^t \beta_2^{t-s}  \left[2\norm{\nabla_{(b)} L(\vx_0)}_2^2 + 2 \norm{\nabla_{(b)} L(\vx_{s-1}) -\nabla_{(b)} L(\vx_0) }_2^2 \right]/d_b \\
    \leq& \beta_2^t \norm{\vv_0}_\infty +2 (1-\beta_2^t) \sum_{b=1}^B \sigma_b^2 + 4(1-\beta_2^t) \max_{b \in [B]} \norm{\nabla_{(b)} L(\vx_0)}_2^2 /d_b
    + 4(1-\beta_2) \E  \sum_{s=1}^t \beta_2^{t-s} \norm{ \nabla L(\vx_{s-1}) - \nabla L(\vx_0) }_2^2\\
    \leq& \beta_2^t \norm{\vv_0}_\infty +2 (1-\beta_2^t) \sum_{b=1}^B \sigma_b^2 + 4(1-\beta_2^t) \max_{b \in [B]} \norm{\nabla_{(b)} L(\vx_0)}_2^2 /d_b
    + 4(1-\beta_2) \E  \sum_{s=1}^t \beta_2^{t-s} \sum_{b \in [B]} H_b^2 \norm{\vx_{s-1, (b)}-\vx_{0,(b)}}_2^2\\
    \leq &\beta_2^t \norm{\vv_0}_\infty +2 \sum_{b=1}^B \sigma_b^2 + 4 \max_{b \in [B]} \norm{\nabla_{(b)} L(\vx_0)}_2^2 /d_b
    + 4(1-\beta_2) \E  \sum_{s=1}^t \beta_2^{t-s} \sum_{b \in [B]} H_b^2 \norm{\vx_{s-1, (b)}-\vx_{0,(b)}}_2^2. 
\end{align*}
We define $C=v_0 + 2 \sum_{b=1}^B \sigma_b^2 + 4\max_{b \in [B]} \norm{\nabla_{(b)} L(\vx_0)}_2^2 /d_b$ for simplicity. 
From \Cref{lem:momentum_ratio} and Cauchy inequality, we know that 
\begin{align*}
    \frac{1}{d_{b}} \norm{\vx_{t, (b)}-\vx_{0,(b)}}_2^2 &= \frac{\eta^2}{d_{b}} \sum_{\Phi(j)=b} \abs{\sum_{s=1}^t \frac{g_{s,j}}{\sqrt{v_{s,b} +\epsilon}}}^2 \\
    &\leq \frac{\eta^2}{d_{b}} \sum_{\Phi(j)=b} t \sum_{s=1}^t\frac{g_{s,j}^2}{v_{s,b}+\epsilon}\\
    &= \eta^2 t \sum_{s=1}^t \frac{\sum_{\Phi(j)=b} g_{s,j}^2/d_{b}}{v_{s,b}+\epsilon}\\
    &\leq \eta^2 t \left(t + \frac{\beta_2}{1-\beta_2} \ln\frac{v_{t,b}+\epsilon}{v_{0,b}+\epsilon}  \right)\\
    &\leq \eta^2 t^2 +\eta^2 t \frac{\beta_2}{1-\beta_2} \ln\frac{\max_{b' \in [B]} v_{t,b'}+\epsilon}{v_0 +\epsilon}.
\end{align*}
So we can get that 
\begin{align*}
    &\E \sum_{b \in [B]} H_b^2 \norm{\vx_{t, (b)}-\vx_{0,(b)}}_2^2 \\
    \leq & \eta^2 t^2 \sum_{b \in [B]} H_b^2 d_b+\eta^2 t \frac{\beta_2}{1-\beta_2} \sum_{b \in [B]} H_b^2 d_b \E \ln\frac{\max_{b' \in [B]} v_{t,b'}+\epsilon}{v_0 +\epsilon}\\
    \leq & \eta^2 t^2 \sum_{b \in [B]} H_b^2 d_b+\eta^2 t \frac{\beta_2}{1-\beta_2} \sum_{b \in [B]} H_b^2 d_b \ln\frac{ \E \max_{b' \in [B]} v_{t,b'}+\epsilon}{v_0 +\epsilon}. 
\end{align*}
Define $G= \max_{1\leq t \leq T} \E \max_{b \in [B]} v_{t,b} +\epsilon$. There exists $t \leq T$ such that 
\begin{align*}
    G&=\E \max_{b \in [B]} v_{t,b} + \epsilon\\
    &\leq \epsilon+C+ 4(1-\beta_2)\E \sum_{s=1}^{t-s} \beta_2^{t-s}\sum_{b \in [B]} H_b^2 \norm{\vx_{s-1, (b)} - \vx_{0, (b)}}_2^2\\
    &\leq \epsilon+ C+ 4(1-\beta_2)\sum_{s=1}^{t-s} \beta_2^{t-s} \left( \eta^2 (s-1)^2 \sum_{b \in [B]} H_b^2 d_b+\eta^2 (s-1) \frac{\beta_2}{1-\beta_2} \sum_{b \in [B]} H_b^2 d_b\ln\frac{G}{v_0 +\epsilon}\right)\\
    &\leq \epsilon + C + 4 \eta^2 T^2 \sum_{b \in [B]} H_b^2 d_b + 4 \eta^2 T \frac{\beta_2}{1-\beta_2} \sum_{b \in [B]} H_b^2 d_b \ln\frac{G}{v_0 +\epsilon}\\
    &\leq \epsilon+ C + 4\eta^2 T^2 \sum_{b \in [B]} H_b^2 d_b\\
    & \phantom{\leq \epsilon }+ 4 \eta^2 T \frac{\beta_2}{1-\beta_2} \sum_{b \in [B]} H_b^2 d_b\left(\ln{\frac{G(1-\beta_2)}{4 \sum_{b \in [B]} H_b^2 d_b \eta^2 T \beta_2 }} + \ln{4 \sum_{b \in [B]} H_b^2 d_b \eta^2 T \frac{\beta_2}{(1-\beta_2)(v_0+\epsilon)} }\right)\\
    &\leq \epsilon+ C + 4 \eta^2 T^2 \sum_{b \in [B]} H_b^2 d_b+ \frac{G}{2} + 4 \eta^2 T \frac{\beta_2}{1-\beta_2} \sum_{b \in [B]} H_b^2 d_b\ln{4 \sum_{b \in [B]} H_b^2 d_b \eta^2 T \frac{\beta_2}{(1-\beta_2)(v_0+\epsilon)} }\\
    &\leq \epsilon+ C + 4 \eta^2 T^2 \sum_{b \in [B]} H_b^2 d_b + \frac{G}{2} + (v_0+\epsilon)\left(4 \sum_{b \in [B]} H_b^2 d_b \eta^2 T \frac{\beta_2}{(1-\beta_2)(v_0+\epsilon)}\right)^2.
\end{align*}
The last two inequalites come from $\ln{x} \leq \frac{x}{2}$ and $x\ln(x) \leq x^2$. Then we can get that 
\begin{align*}
G\leq 2\epsilon + 2C + 8 \sum_{b \in [B]} H_b^2 d_b \eta^2 T^2 + 32(v_0+\epsilon)\left( \sum_{b \in [B]} H_b^2 d_b \eta^2 T \frac{\beta_2}{(1-\beta_2)(v_0+\epsilon)}\right)^2 
\end{align*}
and
\begin{align*}
&\ln{\frac{\E \max_{b \in [B]} v_{T,b}+\epsilon}{v_0+\epsilon}}\\
   \leq &\ln{\frac{2\epsilon + 2C + 8\sum_{b \in [B]} H_b^2 d_b \eta^2 T^2 + 32(v_0+\epsilon)\left( \sum_{b \in [B]} H_b^2 d_b  \eta^2 T \frac{\beta_2}{(1-\beta_2)(v_0+\epsilon)}\right)^2}{v_{0}+\epsilon}}\\
   \leq& \ln{\left[ 2\left(1 +\frac{\sum_{b=1}^B \sigma_b^2 + 2\norm{\nabla L(\vx_0)}_\Phi^2}{v_0+\epsilon} + \frac{2\sum_{b \in [B]} H_b^ 2 d_b \eta^2 T^2}{v_0+\epsilon} + \frac{4\sum_{b \in [B]} H_b^2 d_b \eta^2 T \beta_2}{(1-\beta_2)(v_0+\epsilon)}\right)^2 \right]}\\
   \leq& 2\ln\left(1+ \frac{ \sum_{b=1}^B \sigma_b^2 + \norm{\nabla L(\vx_0)}_{\Phi}^2 + \sum_{b\in[B]} H_b^2 d_b
   \eta^2 T(T + \frac{1}{1-\beta_2})} {v_0+\epsilon}\right) + \ln{32}
\end{align*}
% \shuo{Now it becomes $O(\log(dT))$. }
\end{proof}

Finally, we give the proof for \Cref{thm:main_general_norm}. When $\Phi(i)=i$, i.e., each parameter forms a single block, it becomes the proof for \Cref{thm:main}. 
\maingeneral*
\begin{proof}[Proof of \Cref{thm:main_general_norm}]
From \Cref{defi:smoothness_to_partition}, there exists a diagonal matrix $\mH$ that follows $\Phi$ and always dominates $\nabla^2 L(\vx)$ satisfying $H(L,\Phi) = \Tr(\mH)=\sum_{b \in [B]} H_b d_b$. In a single step, we can have that
    \begin{align*}
        L(\vx_{t})-L(\vx_{t-1})&\leq \nabla L(\vx_{t-1})^\top(\vx_{t}-\vx_{t-1}) + \frac{1}{2}\sum_{b=1}^B H_b \sum_{\Phi(i)=b} \left(x_{t,i}-x_{t-1,i}\right)^2\\
        &=-\eta \sum_{i=1}^d \frac{g_{t,i} \bar{g}_{t,i}}{\sqrt{v_{t,\Phi(i)}+\epsilon}} +\frac{1}{2}\eta^2 \sum_{b=1}^B H_b \frac{\norm{\vg_{t,(b)}}_2^2}{v_{t,b}+\epsilon}. 
    \end{align*}
    If we sum over $t$ from $1$ to $T$ and take expectation, we can get 
    \begin{align*}
        \E \left[L(\vx_{T})-L(\vx_0)\right]
    &\leq -\E\left[\eta \sum_{i=1}^d \sum_{t=1}^T \frac{g_{t,i}\bar{g}_{t,i}}{\sqrt{v_{t,\Phi(i)}+\epsilon}} \right]+\frac{1}{2}\eta^2 \E \left[\sum_{b=1}^B H_b \sum_{t=1}^T \frac{\norm{\vg_{t,(b)}}_2^2}{v_{t,b}+\epsilon}\right]\\
    &\leq -\E\left[\eta \sum_{i=1}^d \sum_{t=1}^T \frac{g_{t,i}\bar{g}_{t,i}}{\sqrt{v_{t,\Phi(i)}+\epsilon}} \right]+\frac{1}{2}\eta^2 \E \left[\sum_{b=1}^B H_b d_b \left(T+\frac{\beta_2}{1-\beta_2}\ln{\frac{v_{T,b}+\epsilon}{v_{0,b}+\epsilon}} \right)\right]. 
    \end{align*}
    The second inequality comes from applying \Cref{lem:momentum_ratio}. By \Cref{lem:first_order_approx_rmsprop}, we have that 
    \begin{align*}
        \frac{1}{T}\E \left[ \sum_{i=1}^d \sum_{t=1}^T \frac{\bar{g}_{t,i}^2}{\sqrt{\tilde{v}_{t,\Phi(i)}+\epsilon}}\right] &\leq \frac{2}{\eta T} \E \left[L(\vx_0)-L(\vx_T) \right] + \frac{\eta}{T} \E \left[\sum_{b=1}^B H_b d_b  \left( T+\frac{\beta_2}{1-\beta_2}\ln{\frac{v_{T,b}+\epsilon}{v_{0,b}+\epsilon}} \right) \right]\\
        &+\frac{1}{T}\sum_{b=1}^B d_b\sigma_b \sqrt{1-\beta_2}\left(T+\frac{\beta_2}{1-\beta_2} \E \ln{\frac{v_{T,b}+\epsilon }{v_{0,b}+\epsilon}}\right)\\
        &\leq \frac{2}{\eta T} \E \left[L(\vx_0)-L(\vx_T) \right] + \eta \sum_{b=1}^B H_b d_b + \sqrt{1-\beta_2} \sum_{b=1}^B d_b \sigma_b\\
        &+\frac{\beta_2}{T(1-\beta_2)} \left( \eta \sum_{b=1}^B H_b d_b + \sqrt{1-\beta_2} \sum_{b=1}^B \sigma_b\right)\max_{b \in [B]} \E \ln{\frac{v_{T,b}+\epsilon}{v_{0,b}+\epsilon}} \\
        &\leq \frac{2}{\eta T} \E \left[L(\vx_0)-L(\vx_T) \right] + \eta \sum_{b=1}^B H_b d_b+ \sqrt{1-\beta_2} \sum_{b=1}^B d_b \sigma_b\\
        &+\frac{\beta_2}{T(1-\beta_2)} \left( \eta \sum_{b=1}^B H_b d_b+ \sqrt{1-\beta_2} \sum_{b=1}^B d_b \sigma_b \right)\ln{ \frac{ \E \max_{b \in [B]} v_{T,b}+\epsilon}{v_{0}+\epsilon}} 
    \end{align*}
From \Cref{lem:vt_growth}, we can define 
\begin{align*}
    E&=\frac{2}{\eta T} \E \left[L(\vx_0)-L(\vx_T) \right] + \eta \sum_{b=1}^B H_b d_b + \sqrt{1-\beta_2} \sum_{b=1}^B d_b \sigma_b
        +\frac{\beta_2}{T(1-\beta_2)} \left( \eta \sum_{b=1}^B H_b d_b+ \sqrt{1-\beta_2} \sum_{b=1}^B d_b \sigma_b\right)F,
\end{align*}
with 
\begin{align*}
    F=2\ln\left(1+ \frac{ \sum_{b=1}^B \sigma_b^2 + \norm{\nabla L(\vx_0)}_{\Phi}^2 + \sum_{b\in[B]} H_b^2 d_b \eta^2 T(T + \frac{1}{1-\beta_2})} {v_0+\epsilon}\right) + \ln{32}. 
\end{align*}
Then it holds that
\begin{align*}
    \frac{1}{T}\E \left[ \sum_{i=1}^d \sum_{t=1}^T \frac{\bar{g}_{t,i}^2}{\sqrt{\tilde{v}_{t,\Phi(i)}+\epsilon}}\right] \leq E. 
\end{align*}
    By \Cref{lem:denominator} and Cauchy inequality, we have that 
    \begin{align*}
        \frac{2}{T}\E \sum_{t=\frac{T}{2}+1}^T \sum_{b=1}^B \sqrt{d_b} \norm{\bar{\vg}_{t,(b)} }_2 &\leq \left(\frac{2}{T} \E \sum_{t=\frac{T}{2}+1}^T \sum_{b=1}^B  \frac{\norm{\bar{\vg}_{t,(b)}}_2^2}{\sqrt{\tilde{v}_{t,b}+\epsilon}}\right)^\frac{1}{2} \left(\frac{2}{T} \E \sum_{t=\frac{T}{2}+1}^T \sum_{b=1}^B d_b \sqrt{\tilde{v}_{t,b}+\epsilon} \right)^\frac{1}{2} \\
        &\leq \left(\frac{2}{T} \E \sum_{t=\frac{T}{2}+1}^T \sum_{b=1}^B \sum_{\Phi(i)=b} \frac{\bar{g}_{t,i}^2}{\sqrt{\tilde{v}_{t,b}+\epsilon}}\right)^\frac{1}{2} \left(\frac{2}{T} \E \sum_{t=\frac{T}{2}+1}^T \sum_{b=1}^B d_b \sqrt{\tilde{v}_{t,b}+\epsilon} \right)^\frac{1}{2}\\
        &\leq \sqrt{2 E} \left(4 E + \frac{4 \beta_2^\frac{T}{4}}{T(1-\beta_2)} d \sqrt{v_{0}} + \sum_{b=1}^B d_b \sigma_b + d \sqrt{\epsilon} \right)^\frac{1}{2}\\
        &\leq 2\sqrt{2} E + \sqrt{2}\sqrt{E} \sqrt{\frac{4 \beta_2^\frac{T}{4}}{T(1-\beta_2)} d \sqrt{v_{0}} + \sum_{b=1}^B d_b \sigma_b + d \sqrt{\epsilon}}.
    \end{align*}
This completes the proof.
\end{proof}

\section{Experiment Details}\label{sec:exp_details}

\subsection{\texorpdfstring{\adam}{Adam} on a rotated loss}\label{subsec:generating_random_matrix}
A key difficulty in implementing \rotatedadam arises from applying an orthogonal rotation on the parameters before calculating the loss. It is computationally infeasible to apply a 125M $\times$ 125M orthogonal matrix on the 125M-sized parameter vector. To avoid such computation, we design a new orthogonal transformer to rotate the parameters of the network. In what follows, we elaborate on this rotation.

\paragraph{\texorpdfstring{\randperm.}{RandPerm}} Given a vector $v$ of size $d$, we can orthogonally rotate it by repeatedly applying these consecutive operations: 1. Permute the entries of the vector according to a randomly chosen permutation $\pi \in \mathbb{S}_d$. 2. Reshape the permuted vector into a 3D tensor of size $[s_1, s_2, s_3]$, apply a fixed orthogonal rotation of size $s \times s$ on each side of the tensor and then reshape it back to a vector of size $d$.

This operation performs an orthogonal transformation $\mathcal{R}$ on the input vector $v$. We can chain multiple operations of this kind and construct $\text{\randperm}^k$, where $k$ is a positive number indicating the number of consecutive \randperm~s applied. Building upon this rotation, we train GPT-2 125M with \adam on $L \circ \text{\randperm}^2$ to analyze our hypothesis regarding the $\ell_\infty$ geometry of the loss landscape and to verify that \adam will indeed suffer from the induced orthogonal equivariance. \Cref{fig:adam_results} confirms our findings, as the performance of \rotatedadam with $\text{\randperm}^2$ is significantly worse than \adam. This suggests that \adam is highly sensitive to the rotation and adaptivity alone can't explain its advantage.

\subsection{Computation of matrix norms}\label{subsec:exp_matrix_norm_estimation}
As mentioned in \Cref{sec:general_theory}, it is computationally infeasible to get the full Hessian matrix and directly compute norms of it. Instead we leverage Hessian vector product function in Jax to probe the Hessian matrix. We use Lanczos algorithm to estimate spectral norm. %\shuo{do we use 20 or 50 iterations?}

We propose \Cref{alg:11norm} to estimate the sum of absolute values for each row and sum over all the rows to get $(1,1)$-norm of Hessian matrix. We first subsample a fixed batch of training data for estimating the $(1,1)$-norm of Hessian matrix. The high-level idea is to compute the matrix vector products between Hessian of training loss on this batch and a sequence of random Cauchy vectors. Then we take the $\ell_1$ norm of the coordinate-wise median of the resulting sequence of Hessian vector products. Because the Cauchy distribution is 1-stable, the resulting product is also a vector of Cauchy random variables, and the magnitude of each element equals to $\ell_1$ norm of the corresponding row of the Hessian. Thus with infinitely many samples, the $\ell_1$ norm of the coordinate-wise median converges almost surely to the $(1,1)$-norm of the Hessian. 

We choose $n=200$ for the measurement experiments on GPT-2 and $n=50$ for the measurement experiments on ResNet18. 
We also prove a non-asymptotic high-probability multiplicative bound for the estimation error which depends mildly on the dimension $d$ in \Cref{thm:measure_11_norm}. 

\algrenewcommand\algorithmicrequire{\textbf{Input:}}
\begin{algorithm}
\caption{Estimation of $(1,1)$-Norm of Hessian, $\nabla^2L(\vx)$}
\begin{algorithmic}[1]\label{alg:11norm}
\Require Number of Cauchy vectors $n$, parameter $\vx \in \mathbb{R}^d$, loss $L$

\For{$i = 1$ to $n$}
     \State Sample a independent Cauchy vector $v^{(i)} \in \R^d$ where $v^{(i)}_j \overset{\textup{i.i.d.}}{\sim} \text{Cauchy}(0,1)$ for $j=1,\ldots,d$.  
    \State $\mathbf{H}_{:,i}  \gets \nabla^2 L (\vx) \cdot v^{(i)}$ \hfill (Using hessian-vector product) 
\EndFor
\State \Return $\sum_{j=1}^d \mathrm{median}(\abs{\mathbf{H}_{j,:}})$
\end{algorithmic}
\end{algorithm}

    First we prove that median of random variables following uniform distribution is sub-Gaussian. 
    \begin{lemma}\label{lem:median_concentraion}
        Suppose $Z_1, \cdots, Z_n \overset{\mathrm{iid}}{\sim} \textup{Unif}([0,1])$. Then $P(\abs{\textup{median}(Z_1, \cdots, Z_n) -\frac{1}{2}} \geq \epsilon) \leq 2 \exp{\left(-2n\epsilon^2 \right)}$ for any $\epsilon \geq 0$. 
    \end{lemma}
    \begin{proof}[Proof of \Cref{lem:median_concentraion}]
        Define $S=\sum_{i=1}^n \mathbf{1}_{Z_i \leq \frac{1}{2}-\epsilon}$. Since $\mathbf{1}_{Z_i \leq \frac{1}{2}-\epsilon}$ follows i.i.d. Bernoulli distribution with $p_1=\frac{1}{2}-\epsilon$, $S \sim \textup{Bin}(n, p_1)$. 

$M_n=\textup{median}(Z_1, \cdots, Z_n) \leq \frac{1}{2}-\epsilon$ if and only if at least $\frac{n+1}{2}$ $Z_i$'s are smaller than $\frac{1}{2}-\epsilon$. And we can apply Hoeffding's inequality on $S$ and get that
\begin{align*}
    P(M_n \leq \frac{1}{2}-\epsilon)&=P(S \geq \frac{n+1}{2}) \leq P(S \geq \frac{n}{2}) \\
    &=P(S-np_1 \geq \frac{n}{2}-np_1)\\
    &\leq \exp{\left(-\frac{2 (\frac{n}{2}-np_1)^2 }{n} \right)}\\
    &= \exp{\left(-2n\epsilon^2 \right)}. 
\end{align*}
We can know $P(M_n \geq \frac{1}{2}+\epsilon) \leq \exp(-2n\epsilon^2)$ from the symmetry of distribution. 
    \end{proof}

\normestimate*
\begin{proof}[Proof of \Cref{thm:measure_11_norm}]
    Define $a_j = \sum_{k=1}^d \abs{\nabla^2 L(\vx)_{j,k}}$ for $j \in [d]$. When $v_j^{(i)} \overset{\textup{i.i.d.}}{\sim} \textup{Cauchy}(0,1)$, it holds that $H_{j,i}=\nabla^2 L(\vx)_{j,:} \cdot v^{(i)}$ follows $\textup{Cauchy}(0, a_j)$ because Cauchy distribution is $1$-stable. And $H_{j,1}, \cdots, H_{j,n}$ are independent. Then it suffices to show that 
    \begin{align}
  P\left(\abs{\sum_{j=1}^d  a_j \textup{median}(\abs{Y_{j,1}}, \cdots, \abs{Y_{j,n}}) - \sum_{j=1}^d a_j}\geq \epsilon \sum_{j=1}^d a_j \right)\leq 2d \exp(-\frac{n\Delta^2}{2}) +  2\exp{\left(-\frac{2n\epsilon^2 \cos^4((1+\Delta)\frac{\pi}{4})}{\pi^2} \right)}.       
    \end{align}
    for any $\{Y_{j,k}\}$ such that $Y_{j,1}, \cdots, Y_{j,n} \overset{\mathrm{iid}}{\sim} \textup{Cauchy}(0,1)$ for any $j \in [d]$. Furthermore, $(\abs{Y_{j,1}}, \cdots, \abs{Y_{j,n}})\overset{d}{=}(\tan(X_{j,1}), \cdots, \tan(X_{j,n}) )$ for $X_{j,1}, \cdots, X_{j,n} \overset{\mathrm{iid}}{\sim} \textup{Unif}(0, \frac{\pi}{2})$. So we only need to show that 
    \begin{align}\label{eq:concentration_goal}
  P\left(\abs{\sum_{j=1}^d  a_j \textup{median}(\tan(X_{j,1}), \cdots, \tan(X_{j,n})) - \sum_{j=1}^d a_j}\geq \epsilon \sum_{j=1}^d a_j \right)\leq 2d \exp(-\frac{n\Delta^2}{2}) +  2\exp{\left(-\frac{2n\epsilon^2 \cos^4((1+\Delta)\frac{\pi}{4})}{\pi^2} \right)}.       
    \end{align}
    for any $\{ X_{j,k}\}$ such that $X_{j,1}, \cdots, X_{j,n} \overset{\mathrm{iid}}{\sim} \textup{Unif}(0, \frac{\pi}{2})$ for any $j \in [d]$. 

    Fix $\Delta \in (0,1)$. We define 
    \[
f(x) =
\begin{cases}
  \tan(x) & \text{if } \abs{x-\frac{\pi}{4}} \leq \Delta \frac{\pi}{4}, \\
  \frac{1}{\cos^2((1-\Delta) \frac{\pi}{4})} (x-(1-\Delta)\frac{\pi}{4}) + \tan((1-\Delta)\frac{\pi}{4}) & \text{if } 0<x <(1-\Delta) \frac{\pi}{4}, \\
  \frac{1}{\cos^2((1+\Delta) \frac{\pi}{4})} (x-(1+\Delta)\frac{\pi}{4}) + \tan((1+\Delta)\frac{\pi}{4}) & \text{if } (1+\Delta) \frac{\pi}{4} <x < \frac{\pi}{2}.
\end{cases}
\]
Then $f(x)$ is a differentiable function on $(0, \frac{\pi}{2})$ that equals to $\tan(x)$ in the middle and is linear on both ends. 

We can decompose \Cref{eq:concentration_goal} into 
\begin{align*}
    &P\left(\abs{\sum_{j=1}^d  a_j \textup{median}(\tan(X_{j,1}), \cdots, \tan(X_{j,n})) - \sum_{j=1}^d a_j}> \epsilon \sum_{j=1}^d a_j \right)\\
    =&P\left(\abs{\sum_{j=1}^d  a_j \tan(\textup{median}(X_{j,1}, \cdots, X_{j,n})) - \sum_{j=1}^d a_j}> \epsilon \sum_{j=1}^d a_j \right)\\
    \leq& P\left(\abs{\sum_{j=1}^d  a_j \tan(\textup{median}(X_{j,1}, \cdots, X_{j,n})) - \sum_{j=1}^d a_j f(\textup{median}(X_{j,1}, \cdots, X_{j,n}))}> 0\right)\\
    +& P\left(\abs{\sum_{j=1}^d  a_j f(\textup{median}(X_{j,1}, \cdots, X_{j,n})) - \sum_{j=1}^d a_j }> \epsilon \sum_{j=1}^d a_j \right). 
\end{align*}
For the first part, we have that 
\begin{align*}
  % &P\left(\abs{\sum_{j=1}^d  a_j \tan(\textup{median}(X_{j,1}, \cdots, X_{j,n})) - \sum_{j=1}^d a_j f(\textup{median}(X_{j,1}, \cdots, X_{j,n}))}> \frac{\epsilon}{2} \sum_{j=1}^d a_j \right)  \\
  & P\left(\abs{\sum_{j=1}^d  a_j \tan(\textup{median}(X_{j,1}, \cdots, X_{j,n})) - \sum_{j=1}^d a_j f(\textup{median}(X_{j,1}, \cdots, X_{j,n}))}> 0 \right)\\
  \leq & \sum_{j=1}^d P\left(\abs{\tan(\textup{median}(X_{j,1}, \cdots, X_{j,n})) - f(\textup{median}(X_{j,1}, \cdots, X_{j,n}))}> 0 \right)\\
  \leq & \sum_{j=1}^d P\left(\abs{\textup{median}(X_{j,1}, \cdots, X_{j,n}) - \frac{\pi}{4}} > \Delta \frac{\pi}{4} \right)\\
  \leq & 2d \exp{(-\frac{n\Delta^2}{2})} 
\end{align*}
where we apply \Cref{lem:median_concentraion} in the last step. 

For the second part, we first know that $\textup{median}(X_{j,1}, \cdots, X_{j,n})-\frac{\pi}{4}$ is sub-Gaussian variable with $\sigma^2 = \frac{\pi^2}{16n}$ from \Cref{lem:median_concentraion}. Then $f(\textup{median}(X_{j,1}, \cdots, X_{j,n})) - \E f(\textup{median}(X_{j,1}, \cdots, X_{j,n}))$ is sub-Gaussian variable with $\sigma^2 = \frac{1}{\cos^4((1+\Delta)\frac{\pi}{4})} \frac{\pi^2}{16n}$ because $f'(x) \leq \frac{1}{\cos^2((1+\Delta)\frac{\pi}{4})}$. And $\sum_{j=1}^d a_j f(\textup{median}(X_{j,1}, \cdots, X_{j,n})) - \sum_{j=1}^d a_j \E f(\textup{median}(X_{j,1}, \cdots, X_{j,n}))$ is sub-Gaussian variable with $\sigma^2 = \frac{1}{\cos^4((1+\Delta)\frac{\pi}{4})} \frac{\pi^2}{16n} \left(\sum_{j=1}^d a_j\right)^2 $. 

When $n \geq \frac{\pi^3}{2\epsilon^2 \cos^4((1+\Delta)\frac{\pi}{4})}$, it holds that
\begin{align*}
    \abs{\E f(\textup{median}(X_{j,1}, \cdots, X_{j,n})) - 1}&\leq \E \abs{f(\textup{median}(X_{j,1}, \cdots, X_{j,n})) - f(\frac{\pi}{4})}\\
    &\leq \max_{x} f'(x)\E \abs{\textup{median}(X_{j,1}, \cdots, X_{j,n}) -\frac{\pi}{4}}\\
    &\leq \frac{1}{\cos^2((1+\Delta)\frac{\pi}{4})} \sqrt{2\pi} \frac{\pi}{4 \sqrt{n}} \leq \frac{\epsilon}{2}
\end{align*}
and 
\begin{align*}
 &P\left(\abs{\sum_{j=1}^d  a_j f(\textup{median}(X_{j,1}, \cdots, X_{j,n})) - \sum_{j=1}^d a_j }> \epsilon \sum_{j=1}^d a_j \right)\\
 \leq& P\left(\abs{\sum_{j=1}^d  a_j f(\textup{median}(X_{j,1}, \cdots, X_{j,n})) - \sum_{j=1}^d a_j \E f(\textup{median}(X_{j,1}, \cdots, X_{j,n}))}> \frac{\epsilon}{2} \sum_{j=1}^d a_j \right)\\
 +& \sum_{j=1}^d P\left(\abs{\E f(\textup{median}(X_{j,1}, \cdots, X_{j,n})) - 1}>\frac{\epsilon}{2} \right)\\
 \leq &2\exp{\left(-\frac{2n\epsilon^2 \cos^4((1+\Delta)\frac{\pi}{4})}{\pi^2} \right)}. 
\end{align*}
Combining these two parts, we can get that 
\begin{align*}
&P\left(\abs{\sum_{j=1}^d  a_j \textup{median}(\tan(X_{j,1}), \cdots, \tan(X_{j,n})) - \sum_{j=1}^d a_j}\geq \epsilon \sum_{j=1}^d a_j \right)\\
\leq &2d \exp(-\frac{n\Delta^2}{2}) +  2\exp{\left(-\frac{2n\epsilon^2 \cos^4((1+\Delta)\frac{\pi}{4})}{\pi^2} \right)}
\end{align*}
for $n \geq \frac{\pi^3}{2\epsilon^2 \cos^4((1+\Delta)\frac{\pi}{4})}$. 
    \end{proof}

\end{document}